\documentclass[11pt]{article}

\usepackage{fullpage}
\usepackage{amsmath,amssymb,latexsym,xspace,float,amsthm}
\usepackage{times}
\usepackage{graphicx} 
\usepackage{subfigure} 

\usepackage{natbib}

\usepackage{hyperref}

\usepackage{algorithm}
\usepackage{algorithmic}

\usepackage{multirow}
\usepackage{paralist}

\newtheorem{theorem}{Theorem}
\newtheorem{lemma}{Lemma}
\newtheorem{prep}{Preposition}
\newtheorem{corollary}{Corollary}
\newtheorem{definition}{Definition}
\newtheorem{assumption}{Assumption}

\usepackage{tikz}
\usepackage{paralist}
\usepackage{comment}

\newcommand{\prox}{\operatorname{{prox}}}

\def\b0{{\boldsymbol{0}}}

\def\be{{\boldsymbol{e}}}

\def\bx{{\boldsymbol{x}}}
\def\bs{{\boldsymbol{s}}}

\def\ba{{\boldsymbol{a}}}

\def\bb{{\boldsymbol{b}}}

\def\CDdual-ls{{\sf CDdual-ls}\xspace}

\def\CDdual{{\sf CDdual}\xspace}

\def\liblinear{{\sf LIBLINEAR}\xspace}

\newcommand{\Deltaalpha}{ {\Delta\alpha} }
\def\R{\mathbb{R}}
\def\X{\mathbb{X}}

\def\CCDR1{{\sf CCD++}\xspace}

\def\news{{\sf news20}\xspace}
\def\kddb{{\sf kddb}\xspace}
\def\webspam{{\sf webspam}\xspace}

\def\rcv1{{\sf rcv1}\xspace}

\def\a9a{{\sf a9a}\xspace}
\def\covtype{{\sf covtype}\xspace}

\def\ml1m{{\sf movielens1m}\xspace}

\def\movielens10m{{\sf movielens10m}\xspace}

\newcommand{\bxi}{{\boldsymbol{\xi}}}

\newcommand{\AL}{{\boldsymbol{\alpha}}}
\newcommand{\ALt}{\tilde{{\boldsymbol{\alpha}}}}
\newcommand{\alphat}{\tilde{\alpha}}

\newcommand{\bw}{{\boldsymbol{w}}}
\def\what{\hat{{\boldsymbol w}}}
\def\wbar{\bar{{\boldsymbol w}}}

\newcommand{\bepsilon}{{\boldsymbol{\epsilon}}}

\def\sdcd{{\it DCD}\xspace}
\def\cocoa{{\it CoCoA}\xspace}
\def\ascd{{\it AsySCD}\xspace}
\def\asdcd{{\it PASSCoDe}\xspace} 
\def\atomic{{\it PASSCoDe-Atomic}\xspace} 
\def\wild{{\it PASSCoDe-Wild}\xspace} 
\def\lock{{\it PASSCoDe-Lock}\xspace} 
\def\bfatomic{{\bf PASSCoDe-Atomic}} 
\def\bfwild{{\bf PASSCoDe-Wild}} 
\def\bflock{{\bf PASSCoDe-Lock}}

\def\Z{\mathcal{Z}}
\def\U{\mathcal{U}}
\title{{\bf PASSCoDe}: {\bf P}arallel {\bf AS}ynchronous {\bf S}tochastic \\dual 
{\bf Co}-ordinate {\bf De}scent}
\author{
Cho-Jui Hsieh\\
        {University of Texas, Austin}\\
        {cjhsieh@cs.utexas.edu}
\and
Hsiang-Fu Yu\\
       {University of Texas, Austin}\\
       {rofuyu@cs.utexas.edu}
\and
Inderjit S. Dhillon\\
        {University of Texas, Austin}\\
        {inderjit@cs.utexas.edu}
}
\date{}

\begin{document} 





\maketitle

\begin{abstract} 
  Stochastic Dual Coordinate Descent (\sdcd) has become one of the most efficient
  ways to solve the family of $\ell_2$-regularized empirical risk minimization problems, 
  including linear SVM, logistic regression, 
  and many others. 
  The vanilla implementation of \sdcd is quite slow; however, 
  by maintaining primal variables while updating dual variables, 
  the time complexity of \sdcd can be significantly reduced. Such a strategy forms the core algorithm in the 
  widely-used  LIBLINEAR package. 
  In this paper, we parallelize the \sdcd algorithms in LIBLINEAR. 
  In recent research, several synchronized parallel \sdcd algorithms have been proposed, 
  however, they fail to achieve good speedup in the shared memory multi-core setting. 
  In this paper, we propose a family of asynchronous stochastic dual coordinate descent algorithms (\asdcd).
  Each thread repeatedly selects a random dual variable and conducts coordinate updates using 
  the primal variables that are stored
  in the shared memory. 
  We analyze the convergence properties when different locking/atomic mechanisms are applied. 
  For implementation with atomic operations, we show linear convergence under 
  mild conditions. 
  For implementation without any atomic operations or locking, we present the first {\it backward error analysis}
  for \asdcd under the multi-core environment, showing that the converged solution is the exact solution for 
  a primal problem with perturbed regularizer. 
  Experimental results show that our methods are much faster than previous parallel coordinate
  descent solvers.  
\end{abstract} 

\section{Introduction}
\label{sec:intro}

Given a set of instance-label pairs $(\dot\bx_i, \dot y_i)$, $i=1, \cdots, n$, 
$\dot \bx_i \in \R^d$, $\dot y_i \in \R$, we focus on the following empirical risk 
minimization problem with $\ell_2$-regularization:
\begin{equation}
  \min_{\bw\in \R^d} P(\bw) := \frac{1}{2} \|\bw\|^2 + \sum_{i=1}^n \ell_i(\bw^T \bx_i),  
  \label{eq:primal}
\end{equation}
where $\bx_i = \dot y_i \dot \bx_i$, $\ell_i(\cdot)$ is the loss function and $\|\cdot\|$ is the 2-norm. 
A large class of machine learning problems can be formulated as the above optimization problem. 
Examples include Support Vector Machines (SVMs), logistic regression, ridge regression, and many others. 
Problem~\eqref{eq:primal} is usually called the primal problem, and 
can usually be solved by Stochastic Gradient Descent (SGD)~\citep{ZT04b,SSS07a-short}, 
second order methods~\citep{CJL07a}, or primal coordinate descent algorithms~\citep{KWC07a,FLH09a}. 

Instead of solving the primal problem, another class of algorithms solves the following dual 
problem of \eqref{eq:primal}:
\begin{equation}
  \min_{\AL \in \R^n} D(\AL):= \frac{1}{2} \|\sum_{i=1}^n \alpha_i \bx_i\|^2 + \sum_{i=1}^n \ell_i^*(-\alpha_i), 
  \label{eq:dual}
\end{equation}
where $\ell_i^*(\cdot)$ is the conjugate 
of the loss function $\ell_i(\cdot)$, defined by $\ell^*_i(u) = \max_z (zu-\ell_i(z))$. 
If we define 
\begin{equation}
  \bw(\AL) = \sum_{i=1} \alpha_i \bx_i,
  \label{eq:primal_dual}
\end{equation}
then it is known that  $\bw(\AL^*) = \bw^*$ and $P(\bw^*)=-D(\AL^*)$ where $\bw^*, \AL^*$ are the optimal primal/dual solutions respectively. 
Examples include hinge-loss SVM, square hinge SVM and $\ell_2$-regularized logistic regression. 

Stochastic Dual Coordinate Descent (\sdcd) has become the most widely-used algorithm
for solving \eqref{eq:dual}, and it is faster than primal solvers (including SGD) in many large-scale problems. 
The success of \sdcd is mainly due to the trick of maintaining the primal variables $\bw$ based on the
primal-dual relationship~\eqref{eq:primal_dual}. 
By maintaining $\bw$ in memory, \cite{CJH08a,SSK08a} showed that the time complexity of each 
coordinate update can be reduced from $O(\text{nnz})$ to $O(\text{nnz}/n)$, 
where $\text{nnz}$ is number of nonzeros in the training dataset. 
Several \sdcd algorithms for different machine learning problems are currently implemented in LIBLINEAR~\citep{REF08a-short} 
and they are now widely used in both academia and industry. 
The success of \sdcd has also catalyzed a large body of theoretical studies~\citep{YEN10a,SSS12a}. 

In this paper, we parallelize the \sdcd algorithm in a shared memory multicore system. 
There are two threads of work on parallel coordinate descent. The first thread 
focuses on synchronized algorithms, including synchronized CD~\citep{PR12a,JKB11a-short} and synchronized
\sdcd algorithms~\citep{TY13a,MJ14a}. 
%
%
However, choosing the block size is a trade-off problem between communication and convergence speed, 
so synchronous algorithms usually suffer from slower convergence. 
To overcome this problem, the other thread of work focuses on asynchronous CD algorithms
in multi-core shared memory systems~\citep{JL14a,JL14b}. However, none of the 
existing work maintains both the primal and dual variables. As a result, the 
recent asynchronous CD algorithms end up being much slower than the 
state-of-the-art serial \sdcd algorithms that maintain both $\bw$ and $\AL$, 
as in the \liblinear software. This leads to a challenging question: how to 
maintaining both primal and dual in an asynchronous and efficient way?

In this paper, we propose the first asynchronous dual coordinate descent 
(\asdcd) algorithms with the address to the issue for the primal variable 
maintenance in the shared memory multi-core setting.  
We carefully discuss and analyze three versions of \asdcd: \lock, \atomic, and \wild. 
In \lock, convergence is always guaranteed but the overhead for locking makes it even slower than serial \sdcd. 
In \atomic, the primal-dual relationship~\eqref{eq:primal_dual} is enforced by atomic writes to the shared memory; 
while \wild proceeds without any locking and atomic operations, as a result of 
which the relationship~\eqref{eq:primal_dual} between primal and dual 
variables can be violated due to memory conflicts. 
Our contributions can be summarized below: 
\begin{compactitem}
  \item We propose and analyze a family of asynchronous parallelization of the 
    most efficient \sdcd algorithm: \lock, \atomic, \wild. 
  \item We show linear convergence of \atomic under certain conditions. 
  \item We present a {\em backward error analysis} for \wild and show 
that the converged solution is the exact solution of a primal problem
with a perturbed regularizer.  
Therefore the performance is close-to-optimal on most of the datasets. 
To best of our knowledge, this is the first attempt to analyze a parallel 
machine learning algorithm with memory conflicts using backward error 
analysis, which is a standard tool in numerical analysis~\citep{JHW61a}. 
\item 
Experimental results show that our algorithms 
(\atomic and \wild) are much faster than existing methods. For example, 
on the \webspam dataset, 
\atomic took 2 seconds and \wild took 1.6 seconds to achieve 99\% accuracy,
while \cocoa took 11.5 seconds using 10 threads and \liblinear took 10 seconds using 1 thread to achieve
the same accuracy. 
\end{compactitem}


\section{Related Work}
\label{sec:related}

{\bf Stochastic Coordinate Descent. }
Coordinate descent is a classical optimization technique that has been studied 
for a long time~\citep{DPB95a,ZQL92a}. Recently it has enjoyed renewed interest due
to the success of ``stochastic'' coordinate descent in real applications~\citep{CJH08a,YEN10a}. 
In terms of theoretical analysis, the convergence of (cyclic) 
coordinate descent has been studied for a long time~\citep{ZQL92a,DPB95a}, 
and the global linear convergence is presented recently under certain condition~\citep{SA10a,PWW13a}. 

{\bf Stochastic Dual Coordinate Descent. } 
Many recent papers \citep{CJH08a,HFY10a,SSS12a} have shown that
solving the dual problem using coordinate descent algorithms is faster on large-scale datasets. 
The success of SDCD strongly relies on exploiting the primal-dual relationship \eqref{eq:primal_dual}
to speed up the gradient computation in the dual space. 
\sdcd has become the state-of-the-art solver implemented in \liblinear~\citep{REF08a-short}.
In terms of convergence of dual objective function, some standard theoretical guarantees for coordinate
descent can be directly applied. Different from standard analysis,  
\cite{SSS12a} presented the convergence rate in terms of duality gap.

{\bf Parallel Stochastic Coordinate Descent. }
In order to conduct coordinate updates in parallel, 
\cite{PR12a} studied the algorithm where each processor updates a randomly selected block (or coordinate)
simultaneously, and \cite{JKB11a-short} proposed a similar algorithm for $\ell_1$-regularized problems. 
\cite{CS12a} studied parallel greedy coordinate descent. However, the above synchronized methods
usually face a trade-off in choosing the block size. 
If the block size is small, the load balancing problem leads to slow running time. 
If the block size is large, the convergence speed becomes much slower or the algorithm even diverges.
These problems can be resolved by developing an asynchronous algorithm. 
Asynchronous coordinate descent has been studied by~\citep{DPB89a}, but
they require the Hessian to be diagonal dominant in order to establish the convergence. 
Recently, \cite{JL14b,JL14a} proved linear convergence of asynchronous stochastic coordinate descent
algorithms under the essential strong convexity condition and a ``bounded 
staleness'' condition, where they consider both ``consistent read'' 
and ``inconsistent read'' models. 
\cite{HA14a} showed linear rate of convergence for the asynchronous randomized Gaussian-Seidel
updates, which is a special case of coordinate descent on linear systems. 

{\bf Parallel Stochastic Dual Coordinate Descent. }
For solving~\eqref{eq:subpb-asyn}, 
each coordinate updates only requires the global primal 
variables $\bw$ and one local dual variable $\alpha_i$, 
thus algorithms only need to synchronize $\bw$. 
Based on this observation, \cite{TY13a} proposed to 
update several coordinates or blocks simultaneously and update the global $\bw$, 
and \cite{MJ14a} showed that each block can be solved with other approaches under the same framework. 
However, both these parallel \sdcd methods are synchronized algorithms.

To the best of our knowledge, this is the first to propose and analyze 
asynchronous parallel stochastic dual coordinate descent methods. By 
maintaining a primal solution $\bw$
while updating dual variables, our algorithm is much faster than the previous asynchronous coordinate
descent methods of~\citep{JL14a,JL14b}  
for solving the dual problem~\eqref{eq:dual}. Our algorithms are also faster than 
synchronized dual coordinate descent methods~\citep{TY13a,MJ14a} 
since the latest values of $\bw$ can be accessed by all the threads. 
In terms of theoretical contribution, 
the inconsistent read model in~\citep{JL14a} cannot be directly applied to our algorithm
because each update on $\alpha_i$ is based on the shared $\bw$ vector. 
We further show linear convergence for \atomic, and study the properties
of the converged solution for the {\it wild} version of our algorithm (without any locking and atomic operations)
using a backward error analysis. 
Our algorithm has been successfully applied to solve the collaborative ranking problem~\citep{ITEM_NAME}. 

\vspace{-5pt}
\section{Algorithms}
\label{sec:algorithms}
\vspace{-5pt}
\subsection{Stochastic Dual Coordinate Descent}
We first describe the Stochastic Dual Coordinate Descent (\sdcd) algorithm 
for solving the dual problem~\eqref{eq:dual}. 
At each iteration, \sdcd randomly picks a dual variable $\alpha_i$ and updates it by minimizing the one variable 
subproblem (Eq. \eqref{eq:subpb} in Algorithm 1).
Without exploiting the structure of the quadratic term, the subproblems 
require substantial computation (need $O(\text{nnz})$ time), 
where $\text{nnz}$ is the total number of nonzero elements 
in the training data. 
However, if $\bw(\AL)$ that satisfies~\eqref{eq:primal_dual}
is maintained in memory, 
the subproblem $D(\AL+\delta\be_i)$ can be written as
\begin{align*}
  \vspace{-5pt}  D(\AL + \delta \be_i) =\frac{1}{2}\|\bw+\delta \bx_i\|^2 + \ell_i^*(-(\alpha_i + \delta)), 
\end{align*}
\vspace{-5pt}and the optimal solution can be computed by 
\begin{equation*}
  \delta = \arg\min_{\delta} \frac{1}{2}(\delta + \frac{\bw^T \bx_i}{\|\bx_i\|^2})^2 + 
  \frac{1}{\|\bx_i\|^2}\ell_i^*(-(\alpha_i + \delta)). 
\end{equation*}
Note that all $\|\bx_i\|$ can be pre-computed and are constants. For each coordinate update we only need to solve a simple one-variable subproblem, 
and the main computation is in computing $\bw^T \bx_i$, which requires $O(\text{nnz}/n)$ time. 
For SVM problems, the subproblem has a closed form solution, while 
for logistic regression problems it has to be solved by an iterative solver 
(see~\cite{HFY12c} for details). 
The \sdcd algorithm, which is part of the popular \liblinear package, is described in Algorithm \ref{alg:sdcd}. 

\begin{algorithm}
  \caption{Stochastic Dual Coordinate Descent (\sdcd) }
  \label{alg:sdcd}
  \begin{algorithmic}[1]
    \REQUIRE{Initial $\AL$ and $\bw=\sum_{i=1}^n \alpha_i\bx_i$}
  \WHILE{not converged}
  \STATE Randomly pick $i$

  \STATE Update $\alpha_i\leftarrow \alpha_i + \Delta\alpha_i$, where
  \begin{equation}
    \hspace{-5pt}   \Delta\alpha_i \leftarrow \arg\min_{\delta}  \frac{1}{2}\|\bw + \delta \bx_i\|^2 +\ell^*_i(-(\alpha_i + \delta))
  \label{eq:subpb}
\end{equation}
  \STATE Update $\bw$ by $\bw \leftarrow \bw + \Delta\alpha_i \bx_i$
  \ENDWHILE
 \end{algorithmic}
\end{algorithm}
\vspace{-5pt}
\subsection{Asynchronous Stochastic Dual Coordinate Descent}
\vspace{-5pt}
To parallelize \sdcd in a shared memory multi-core system, 
we propose a family of Asynchronous Stochastic Dual Coordinate Descent (\asdcd) algorithms. 
\asdcd is very simple but effective. Each thread repeatedly run the updates (steps 2 to 4) in Algorithm \ref{alg:sdcd}
using $\bw$, $\AL$, and training data stored in a shared memory. 
The threads do not need to coordinate or synchronize their iterations. 
The details are shown in Algorithm \ref{alg:asdcd}.

Although \asdcd is a simple extension of \sdcd in a multi-core setting, 
there are many options in terms of locking/atomic operations for each step, 
and these choices lead to variations in speed and convergence properties, 
as we will show in this paper. 

\begin{algorithm}
  \caption{Parallel Asynchronous Stochastic dual Co-ordinate Descent (\asdcd)}
  \label{alg:asdcd}
  \begin{algorithmic}
    \REQUIRE{Initial $\AL$ and $\bw=\sum_{i=1}^n \alpha_i\bx_i$}
    \STATE Each thread repeatedly performs the following updates: 
    \STATE \quad step 1: Randomly pick $i$
    \STATE \quad step 2: Update $\alpha_i \leftarrow \alpha_i + \Delta\alpha_i$, 
    where 
  \begin{equation}
   \ \ \ \ \ \ \ \   \Delta\alpha_i \leftarrow \arg\min_{\delta}  \frac{1}{2}\|\bw + \delta \bx_i\|^2 +\ell^*_i(-(\alpha_i + \delta))
  \label{eq:subpb-asyn}
\end{equation}
  \STATE \quad step 3:  Update $\bw$ by $\bw \leftarrow \bw + \Delta\alpha_i \bx_i$
 \end{algorithmic}
\end{algorithm}

Note that the $\Delta\alpha_i$ obtained by subproblem \eqref{eq:subpb-asyn} is 
exactly the same as \eqref{eq:subpb} in Algorithm \ref{alg:sdcd}
if only one thread is involved. However, when there are multiple threads, 
the $\bw$ vector may {\bf not} be the latest one since some other threads have not completed the writes in step 3. 

{\bflock}. To ensure $\bw=\sum_i\alpha_i \bx_i$ for the latest $\AL$, we have to 
lock the following variables between step 1 and 2:
\begin{equation*}
  \text{step 1.5:  \ \ \ lock variables in $N_i:=\{w_t \mid (\bx_i)_t \neq 0\}$. }
\end{equation*}
The locks are then released after step 3. 
With this locking mechanism, \lock will be serializable, i.e., generate the same solution
sequence with the serial \sdcd. 
Unfortunately, threads will waste a lot of time due to the locks, so \lock is 
very slow comparing to the non-locking version (and even slower than 
the serial version of \sdcd). See Table \ref{fig:lock-async} for details. 

\begin{table}
  \caption{ Scaling of \asdcd algorithms. 
  We present the run time (in seconds) for each algorithm on the \rcv1 dataset with 100 iterations, 
  and the speedup of each method over the serial \sdcd algorithm (2x means it is two times faster than the serial algorithm). 
  \label{fig:lock-async}
  }
    \centering
  \begin{tabular}{cccc}
    Number of threads   & Lock & Atomic & Wild \\
   \hline
   \hline
   2   & 98.03s / 0.27x    &  15.28s / 1.75x  &  \bf 14.08s / 1.90x  \\
   \hline
   4  & 106.11s / 0.25x   & 8.35s / 3.20x   &  \bf 7.61s / 3.50x  \\
   \hline
   10  & 114.43s / 0.23x   & 3.86s / 6.91x   & \bf 3.59s / 7.43x \\ 
   \hline
  \end{tabular}
\end{table}


\begin{table}[h]
  \caption{The performance of \wild using $\hat{\bw}$ or $\bar{\bw}$ for prediction. Results show that $\hat{\bw}$ yields much
  better prediction accuracy, which justifies  our theoretical analysis in Section \ref{sec:wild_properties}. \label{tab:w_choice} }
  \centering
  \begin{tabular}{cc|cc|c}
    & & \multicolumn{3}{c}{Prediction Accuracy (\%) by} \\
    & \# threads&  $\hat{\bw}$ & $\bar{\bw}$ & \liblinear\\
    \hline
    \hline
    \multirow{2}{*}{\news}& 4 & 97.1 & 96.1 &  \multirow{2}{*}{97.1} \\
                          & 8 & 97.2 & 93.3 &                        \\
    \hline                                                             
 \multirow{2}{*}{\covtype}& 4 & 67.8 & 38.0 &  \multirow{2}{*}{66.3} \\
                          & 8 & 67.6 & 38.0 &                        \\ 
    \hline                                                              
    \multirow{2}{*}{\rcv1}& 4 & 97.7 & 97.5 &  \multirow{2}{*}{97.7} \\
                          & 8 & 97.7 & 97.4 &                        \\
    \hline                                                             
 \multirow{2}{*}{\webspam}& 4 & 99.1 & 93.1 &  \multirow{2}{*}{99.1} \\
                          & 8 & 99.1 & 88.4 &                        \\
    \hline                                                              
    \multirow{2}{*}{\kddb}& 4 & 88.8 & 79.7 &  \multirow{2}{*}{88.8} \\
                          & 8 & 88.8 & 87.7 &                        
  \end{tabular}
\end{table}
{\bfatomic}. 
The above locking scheme is to ensure that each thread updates $\alpha_i$ based
on the latest $\bw$ values. However, as shown in ~\citep{FN11a-short,JL14a}, 
the effect of using slightly stale values is usually limited in practice. 
Therefore, we propose an \atomic algorithm that avoids locking all the variables in $N_i$ simultaneously. 
Instead, each thread just reads the current $\bw$ values from memory without any locking. 
In practice (see Section~\ref{sec:exp}) we observe that the convergence speed is not significantly affected by using
values of $\bw$. 
However, to ensure that the limit point of the algorithm is still the global optimizer of \eqref{eq:primal}, 
the equation $\bw^*=\sum_i\alpha^*_i\bx_i$
has to be maintained. Therefore, we apply the following ``atomic writes'' in step 3: 
\begin{align*}
  &\text{step 3: For each $j\in N(i)$ }  \\
  &\quad\quad\quad\quad\quad\text{Update $w_{j}\leftarrow w_{j}+ \Delta\alpha_i (\bx_i)_j$ {\bf atomically}}
\end{align*}
\atomic is much faster than \lock as shown in Table \ref{fig:lock-async} since
the atomic writes for a single variable is much faster than locking all the variables. 
However, the convergence of \atomic is not guaranteed by any previous convergence analysis. 
To bridge this gap between practice and theory, 
we prove linear convergence of \atomic under certain conditions in Section \ref{sec:convergence}. 

{\bfwild}.
Finally, we consider Algorithm \ref{alg:asdcd} without any locks and atomic operations. 
The resulting algorithm, \wild, is faster than \atomic and \lock and can achieve almost linear speedup using a single processing unit. 
However, due to the memory conflicts in step 3, some of the
"updates" to $\bw$ will be over-written by other threads. 
As a result, 
the $\hat{\bw}$ and $\hat{\AL}$ outputted by the algorithm usually do not satisfy Eq~\eqref{eq:primal_dual}: 
\begin{equation}
  \hat{\bw} \neq \bar{\bw} := \sum_i \hat{\alpha}_i \bx_i, 
  \label{eq:wbar_define}
\end{equation}
where $\hat{\bw}, \hat{\AL}$ are the primal and dual variables outputted by the algorithm, 
and $\bar{\bw}$ defined in \eqref{eq:wbar_define} is computed from $\hat{\AL}$. 
It is easy to see that $\hat{\AL}$ is not the optimal solution of \eqref{eq:dual}. 
Due to the same reason, in the prediction phase it is not clear whether $\hat{\bw}$ or $\bar{\bw}$ should be used. 
To answer this question, in Section \ref{sec:convergence} we show that $\hat{\bw}$ is actually the optimal solution
of a perturbed primal problem \eqref{eq:primal} using a backward error analysis, where the loss function is the same and the regularization term is slightly perturbed. 
As a result, the prediction should be done using $\hat{\bw}$, and this also yields much better performance in practice, as shown
in Table \ref{tab:w_choice} below.

We summarize the behavior of the three algorithms in Figure~\ref{fig:spectrum}. 
Using locks, the algorithm \lock is serializable but very slow (even slower 
than the serial \sdcd). In the other extreme, the wild version without any lock
and atomic operation has very good speed up, but the behavior can be totally
different from the serial \sdcd. Luckily, in Section \ref{sec:convergence}
we provide the convergence guarantee for \atomic, and apply a backward error
analysis to show that \wild will converge to the solution with the same loss
function with a slightly perturbed regularizer. 

\begin{figure}
\begin{center}
  \begin{tabular}{@{}l@{ }c@{}ccc@{}c}
           &   &Locks       & Atomic Ops & Nothing &\\
   Scaling: & Poor &\multicolumn{3}{c}{
      \begin{tikzpicture}
        \draw[<->,very thick] (-2,0)--(2,0);
      \end{tikzpicture}
    } & Good
    \\
   Serializability: & Perfect &\multicolumn{3}{c}{
      \begin{tikzpicture}
        \draw[<->,very thick] (-2,0)--(2,0);
      \end{tikzpicture}
    } & Poor
\end{tabular}
\end{center}
\caption{Spectrum for the choice of mechanism to avoid memory conflicts for \asdcd.}
\label{fig:spectrum}
\end{figure}

\subsection{Implementation Details}

{\bf Deadlock Avoidance.} Without a proper implementation, the 
deadlock can arise in \lock because a thread needs to acquire all the locks 
associated with $N_i$. A simple way to avoid deadlock is by associating an 
ordering for all the locks such that each thread follows the same ordering to 
acquire the locks. 

\noindent {\bf Random Permutation.} In \liblinear, the random sampling (step 2) of 
Algorithm \ref{alg:sdcd} is replaced by the index from a random permutation, 
such that each $\alpha_i$ can be selected in $n$ steps in stead of $n\log n$ 
steps in expectation. Random permutation can be easily implemented 
asynchronously for Algorithm \ref{alg:asdcd} as follows. Initially, 
given $p$ threads, $\{1,\dots,n\}$ is randomly partitioned into $p$ blocks.
Then, each thread can asynchronously generate the random permutation on its own block 
of variables. 

\noindent {\bf Shrinking Heuristic.} For loss such as hinge and squared-hinge, the 
optimal $\AL^*$ is usually sparse. Based on this property, a shrinking 
strategy was proposed by \cite{CJH08a} to further speed up \sdcd. This 
heuristic is also implemented in \liblinear. The idea is to maintain an active 
set by skipping variables which tend to be fixed. This heuristic can also be 
implemented in Algorithm \ref{alg:asdcd} by maintaining an active set for each 
thread.  

\noindent {\bf Thread Affinity.} The memory design of most modern multi-core machines 
is {\em non-uniform memory access (NUMA)}, where a core has faster memory 
access to its local memory socket. To reduce possible latency due to the 
remote socket access, we should bind each thread to a physical core and 
allocate data in its local memory. Note that the current OpenMP does not 
support this functionality for thread affinity. Library such as 
libnuma 
can be used to enforce 
thread affinity. 
\section{Convergence Analysis}
\label{sec:convergence}
In this section we formally analyze the convergence properties of our proposed algorithms
in Section \ref{sec:algorithms}.
Note that all the proofs can be found in the Appendix.  
We assign a global counter $j$ for the total number of updates, 
and the index $i(j)$ denotes the component selected at step $j$. We define $\{\AL^1, \AL^2, \dots\}$ 
to be the sequence generated by our algorithms, and 
\vspace{-10pt}\begin{equation*}
  \Delta \alpha_j = \alpha^{j+1}_{i(j)} - \alpha^j_{i(j)}. 
\end{equation*}
The update $\Delta\alpha_{j}$ at iteration $j$ is obtained by solving 
\begin{equation*}
  \Delta\alpha_j \leftarrow \arg\min_{\delta} \frac{1}{2}\|\hat{\bw}^j +\delta\bx_{i(j)}\|^2 + \ell^*_{i(j)}(-(\alpha_{i(j)} +\delta)), 
\end{equation*}
where $\what^j$ is the current $\bw$ in the memory. 
We use $\bw^j=\sum_i \alpha_i^{j}\bx_i$ to denote the ``accurate'' $\bw$ 
at iteration $j$. 

In the \lock setting, $\bw^j=\what^j$ is ensured by using the locks. However, 
in \atomic and \wild, $\what^j\neq \bw^j$ because 
some of the updates have not been written into the shared memory. 
To capture this phenomenon, we define $\Z^j$ to be the set of all ``updates to $\bw$'' before iteration $j$: 
\begin{equation*}
  \Z^j := \{(t,k) \mid t<j, k\in N(i(t))\}, 
\end{equation*}
where $N(i(t)):=\{u\mid X_{i(t),u}\neq 0\}$ is all nonzero features in $\bx_{i(t)}$.
We define $\U^j\subseteq \Z^j$ to be the updates that have already been written into $\what^j$. 
Therefore, we have 
\begin{equation*}
  \what^j = \sum_{(t,k)\in \U^j} (\Delta\alpha_t)X_{i(t),k}\be_k. 
\end{equation*}
\vspace{-5pt}
\subsection{Linear Convergence of \atomic}

In \atomic, we assume all the updates before the $(j-\tau)$-th iteration has been written
into $\what^j$, therefore, 
\begin{assumption}
  The set $\U^j$ satisfies $\Z^{j-\tau}\subseteq \U^j \subseteq \Z^j$. 
\end{assumption}
Now we define some constants used in our theoretical analysis. Note that $X\in \R^{n\times d}$
is the data matrix, and we use $\bar{X}\in \R^{n\times d}$
to denote the normalized data matrix where each row is $\bar{\bx}_i^T=\bx_i^T/\|\bx_i\|^2$. 
We then define 
\begin{equation*}
  M_i = \max_{S\subseteq [d]} \|\sum_{t\in S} \bar{X}_{:,t} X_{i,t}\|, \ \ M=\max_i M_i, 
\end{equation*}
where $[d]:=\{1,\dots,d\}$ is the set of all the feature indices, and $\bar{X}_{:,t}$ is the $t$-th column of $\bar{X}$. 
We also define $L_{max}$ to be the Lipschitz constant of $D(\cdot)$ within the level set $\{\AL\mid D(\AL)\leq D(\AL^0)\}$, 
$R_{min}=\min_i \|\bx_i\|^2$, $R_{max}=\max_i\|\bx_i\|^2$. We assume that 
$R_{max}=1$ and there is no zero training sample, so $R_{min}>0$. 

To prove the convergence of asynchronous algorithms, we first show that the expected step size does not increase 
super-linearly by the following Lemma \ref{lm:rho_decrease}. 
\begin{lemma}
  If $\tau$ is small enough such that
    \begin{equation}
      (6\tau(\tau+1)^2eM)/\sqrt{n} \leq 1, 
    \label{eq:condition_tau}
  \end{equation}
  then \atomic satisfies the following inequality:
\begin{equation}
  E(\|\AL^{j-1}-\AL^{j}\|^2) \leq \rho E(\|\AL^j - \AL^{j+1}\|^2), 
  \label{eq:rho_decrease}
\end{equation}
where 
 $ \rho = (1+\frac{6(\tau+1)eM}{\sqrt{n}})^2$.
\label{lm:rho_decrease}
\end{lemma}
The detailed proof is in Appendix \ref{app:lemma}. 
We use a similar technique as in~\citep{JL14a} to prove this lemma, but the proof is different from~\citep{JL14a} because  
\begin{compactitem}
\item Their ``inconsistent read'' model assumes $\hat{\bw}^j = \sum_i \dot{\alpha}_i \bx_i$ for some $\dot{\AL}$. 
  However, in our case $\what^j$ may not be written in this form due to 
  incomplete updates in step 3 of Algorithm \ref{alg:asdcd}. 
\item In~\citep{JL14a}, each coordinate is updated by $\gamma \nabla_t f(\AL)$ with a fixed step size $\gamma$. We
  consider the case that each subproblem \eqref{eq:subpb} is solved exactly. 
\end{compactitem}
To show the linear convergence of our algorithms, we assume the objective function~\eqref{eq:dual} 
satisfies the following property: 
\begin{definition}
  The objective function \eqref{eq:dual} admits the {\it global error bound}
  if there is a constant $\kappa$ such that
  \begin{equation}
    \|\AL-P_S(\AL)\| \leq \kappa \|T(\AL)-\AL\|, 
    \label{eq:globalbound}
  \end{equation}
  where $P_S(\cdot)$ is the projection to the set of optimal solutions, and  
  $T:R^n\rightarrow R^n$ is the operator defined by 
  \begin{equation*}
    T_t(\AL) = \arg\min_{u}\ D(\AL+(u-\alpha_t)\be_t)\quad \forall t=1,\dots,n.
\end{equation*}
  The objective function satisfies the {\it global error bound from the beginning}  if \eqref{eq:globalbound} holds for all $\AL$ satisfying 
  \begin{equation*}
    D(\AL) \leq D(\AL^0)
  \end{equation*}
  where $\AL^0$ is the initial point. 
  \label{def:global_error}
\end{definition}
This definition is a generalized version of Definition 6 in \citep{PWW13a}. 
We list several important machine learning problems that admit global error bounds: 
\begin{itemize}
  \item Support Vector Machines (SVM) with hinge loss~\citep{BB92a}:
    \begin{align}
      \ell_i(z_i) &= C\max(1-z_i , 0) \nonumber\\
      \ell_i^*(-\alpha_i) &= 
      \begin{cases}
        -\alpha_i & \text{ if } 0\leq \alpha_i \leq C,\\
        \infty & \text{ otherwise.}
      \end{cases}
      \label{eq:hinge_svm}
    \end{align}
  \item Support Vector Machines (SVM) with square hinge loss: 
    \begin{align}
      \ell_i(z_i) &= C\max(1-z_i, 0)^2. \nonumber \\ 
      \ell_i^*(-\alpha_i) &= 
      \begin{cases}
        -\alpha_i + \alpha_i^2/4C &\text{ if } \alpha_i \ge 0, \\
        \infty &\text{ otherwise}.
      \end{cases}
    \end{align}
 \end{itemize}
 Note that $C>0$ is the penalty parameter that controls the weights between loss and regularization. 
 \begin{theorem}
   The Support Vector Machines (SVM) with hinge loss or square hinge loss satisfy the global error bound~\eqref{eq:globalbound}. 
 \end{theorem}
 \begin{proof}
   For SVM with hinge loss, each element of the mapping $T(\cdot)$ can be written as
   \begin{align*}
     T_t(\AL) &= \arg\min_{u}\ D(\AL + (u-\alpha_t)\be_t) \\
     &= \arg\min_{u}\ \frac{1}{2}\|\bw(\AL) + (u-\alpha_t)\bx_t\|^2 +\ell^*(-u)\\
              &= \Pi_{\X}\big(\frac{\bw(\AL)^T\bx_t - 1}{\|\bx_t\|^2}\big)
              = \Pi_{\X}\big(\frac{\nabla_t D(\AL)}{\|\bx_t\|^2}\big),\\
   \end{align*}
   where $\Pi_{\X}$ is the projection to the set $\X$, and for hinge-loss SVM $\X:=[0, \ C]$. 
   Using Lemma 26 in~\citep{PWW13a}, we can show that for all $t=1,\dots,n$
   \begin{align*}
         \big|\alpha_t - \Pi_{\X}\big(\frac{\nabla_t D(\AL)}{\|\bx_t\|^2}\big)\big| \ge &\min(1, \frac{1}{\|\bx_t\|^2})\big|\alpha_t - 
         \Pi_{\X}\big(\nabla_t D(\AL)\big)\big| \\
         \ge &\min(1, \frac{1}{R_{max}^2})\big|\alpha_t - 
         \Pi_{\X}\big(\nabla_t D(\AL)\big)\big| \\
         \ge &\big|\alpha_t - \Pi_{\X}\big(\nabla_t D(\AL)\big)\big|,
   \end{align*}
   where the last inequality is due to the assumption that $R_{max}=1$.
   Therefore, 
   \begin{align*}
     \|\AL- T(\AL)\|_2 &\ge \frac{1}{\sqrt{n}}\|\AL-T(\AL)\|_{1}  \\
     &\ge \frac{1}{\sqrt{n}}\sum_{t=1}^n |\alpha_t - \Pi_{\X}\big(\nabla_t D(\AL)\big)|\\
     &= \frac{1}{\sqrt{n}} \|\nabla^+ D(\AL)\|_1 \\
     &\ge \frac{1}{\sqrt{n}} \|\nabla^+ D(\AL)\|_2 \\
     &\ge \frac{1}{\kappa_0\sqrt{n}} \|\AL - P_S(\AL)\|_2,
   \end{align*}
   where $\nabla^+ D(\AL)$ is the projected gradient defined in Definition 5 
   of~\citep{PWW13a} and $\kappa_0$ is the $\kappa$ defined in Theorem 18 of~\citep{PWW13a}. 
   Thus, with $\kappa = \kappa_0\sqrt{n}$, we obtain that the dual function of 
   the hinge-loss SVM satisfies the global error bound defined in 
   Definition~\ref{def:global_error}. 
   Similarly, we can show that the SVM with squared-hinge loss satisfies the 
   global error bound.  
 \end{proof}

Next we explicitly state
the linear convergence guarantee for \atomic.  
\begin{theorem}
  Assume the objective function \eqref{eq:dual} admits a global error bound from the beginning and
  the Lipschitz constant $L_{max}$ is finite in the level set. If \eqref{eq:condition_tau} holds and
  \begin{equation*}
    1 \geq \frac{2L_{max}}{R_{min}^2}(1+\frac{e\tau M}{\sqrt{n}}) (\frac{\tau^2 M^2 e^2}{n})
  \end{equation*}
  then 
  \atomic has a global linear convergence rate in expectation, that is, 
  \begin{equation}
    E[ D(\AL^{j+1}) ] -  D(\AL^*)  \leq \eta \left(E[ D(\AL^j)] - D(\AL^*)\right) , 
  \end{equation}
  where $\AL^*$ is the optimal solution and 
  \begin{equation}
    \eta = 1-\frac{\kappa}{L_{max}}(1-\frac{2L_{max}}{R_{min}^2}(1+\frac{e\tau M}{\sqrt{n}}) (\frac{\tau^2 M^2 e^2}{n}))
  \end{equation}
  \label{thm:converge_atomic}
\end{theorem}
\begin{table}[tb]
  \caption{Data statistics. $\tilde{n}$ is the number of test instances. 
  ${\bar d}$ is the average nnz per instance.}
  \label{tab:data-stat}
\centering
  \begin{tabular}{lrrrrr}
               &  $n$       & $\tilde{n}$ & $d$       &  ${\bar d}$  & $C$    \\
               \hline
               \hline
    \news      & 16,000     & 3,996       & 1,355,191 &  455.5       & 2      \\
    \covtype   & 500,000    & 81,012      & 54        &  11.9        & 0.0625 \\
    \rcv1      & 677,399    & 20,242      & 47,236    &  73.2        & 1      \\
    \webspam   & 280,000    & 70,000      & 16,609,143&  3727.7      & 1      \\
    \kddb      & 19,264,097 & 748,401     & 29,890,095&  29.4        & 1
  \end{tabular}
  \vspace{-10pt}
\end{table}

\subsection{Backward Error Analysis for \wild}
\label{sec:wild_properties}
In \wild, assume the sequence $\{\AL^j\}$ converges
to $\hat{\AL}$ and $\{\bw^j\}$ converges to $\hat{\bw}$.
Now we show that the dual solution $\hat{\AL}$ and the corresponding
primal variables $\bar{\bw}=\sum_{i=1}^n \hat{\alpha}_i \bx_i$ 
are actually the dual and primal solutions of a perturbed problem: 
\begin{theorem}
  $\hat{\AL}$ is the optimal solution of a perturbed dual problem
\vspace{-10pt}  \begin{equation}
        \hat{\AL} = \arg\min_{\AL}\ D(\AL) 
    - \sum_{i=1}^n \alpha_i \bepsilon^T \bx_i, 
\label{eq:dual_new}
  \end{equation}
  and $\bar{\bw}=\sum_i \hat{\alpha}_i \bx_i$ is the solution of the corresponding primal problem:
  \begin{equation}
    \bar{\bw} =  \arg\min_{\bw}\ \frac{1}{2} \bw^T \bw + \sum_{i=1}^n \ell_i( (\bw-\bepsilon)^T \bx_i ),  
    \label{eq:primal_new}
  \end{equation}
where $\bepsilon \in \R^d$ is given by
$  \bepsilon = \bar{\bw} - \what$.  
  \label{thm:dual_new}
\end{theorem}

\begin{proof}
  By definition, $\hat{\AL}$ is the limit point of \wild. 
  Therefore, $\{\Delta\alpha_i\}\rightarrow 0$ for all $i$. 
  Combining with the fact that $\{\what^j\}\rightarrow \what$, 
  we have 
  \begin{equation*}
    -\what^T \bx_i \in \partial_{\alpha_i} \ell_i^*(-\hat\alpha_i),  \quad \forall i. 
  \end{equation*}
  Since $\what = \bar{\bw}- \bepsilon$, we have
  \begin{align*}
    -(\bar{\bw}-\bepsilon)^T \bx_i &\in \partial_{\alpha_i} \ell_i^*(-\hat\alpha_i),  \quad \forall i  \\
    -\bar{\bw}^T \bx_i &\in \partial_{\alpha_i} \left(\ell_i^*(-\hat\alpha_i) - \hat\alpha_i\bepsilon^T\bx_i\right),  \quad \forall i  \\
    0 &\in \partial_{\alpha_i} \left(\frac{1}{2}\|\sum_{i=1}^n \hat\alpha_i 
    \bx_i\|^2+\ell_i^*(-\hat\alpha_i) - 
    \hat\alpha_i\bepsilon^T\bx_i\right),\quad \forall i
  \end{align*}
  which is the optimality condition of \eqref{eq:dual_new}. 
  Thus, $\hat{\AL}$ is the optimal solution of \eqref{eq:dual_new}. 

  For the second part of the theorem, let's consider the following equivalent 
  primal problem and its Lagrangian:
  \begin{align*}
    \min_{\bw, \bxi}\ \frac{1}{2} \bw^T \bw + \sum_{i=1}^n \ell_i(\xi_i)\quad 
    \text{ s.t. } \xi_i = (\bw-\bepsilon)^T\bx_i\ \forall i=1,\dots,n \\
    L(\bw, \bxi, \AL) := \frac{1}{2} \bw^T \bw + \sum_{i=1}^n 
    \{\ell_i(\xi_i)+\alpha_i (  \xi_i - \bw^T\bx_i + \bepsilon^T\bx_i)\}
  \end{align*}
  The corresponding convex version of the dual function can be derived as follows.
  \begin{align*}
    \hat D(\AL) &= \max_{\bw,\bxi} -L(\bw,\bxi, \AL) \\
    & = \left(\max_{\bw} -\frac{1}{2}\bw^T\bw + \sum_{i=1}^{n} \alpha_i \bw^T\bx_i\right)
    +\sum_{i=1}^n \left(\max_{\xi_i}\ -\ell_i(\xi_i)-\alpha_i\xi_i\right)  - \alpha_i \bepsilon^T\bx_i\\
    &= \frac{1}{2} \|\sum_{i=1}^n \alpha_i \bx_i\|^2 + \sum_{i=1}^n 
    \ell_i^*(-\alpha_i) - \alpha_i\bepsilon^T\bx_i \\
    &= D(\AL) - \sum_{i=1}^{n}\alpha_i\bepsilon^T\bx_i
  \end{align*}
The last second equality comes from 1) the substitution of 
$\bw^*=\sum_{i=1}^T \alpha_i\bx_i$ obtained by setting 
$\nabla_{\bw} -L(\bw,\bxi,\AL)=0$; 2) the definition of the conjugate function 
$\ell_i^*(-\alpha_i)$. 
  Thus, the second part of the theorem follows. 
\end{proof}
Note that $\bepsilon$ is the error caused by the memory conflicts. 
From Theorem \ref{thm:dual_new}, 
$\bar{\bw}$ is the optimal solution of the ``biased''
primal problem \eqref{eq:primal_new}, however, in \eqref{eq:primal_new} the
actual model that fits the loss function should be $\hat{\bw}=\bar{\bw} - \bepsilon$. Therefore after the training process
we should use $\hat{\bw}$ to predict, which is the $\bw$ 
we maintained during the parallel coordinate descent updates.
Replacing $\bw$ by $\bw-\bepsilon$ in \eqref{eq:primal_new}, we have the 
following corollary : 
\begin{corollary}
  $\hat{\bw}$ computed by \wild is the solution of the following perturbed primal problem: 
  \vspace{-5pt}  \begin{equation}
    \hat{\bw} = \arg\min_{\bw} \frac{1}{2}(\bw+\bepsilon)^T (\bw+\bepsilon) + 
    \sum_{i=1}^n \ell_i(\bw^T \bx_i)
    \label{eq:hatw_argmin}
  \end{equation}
\end{corollary}
\vspace{-5pt}  The above corollary shows that the computed primal solution $\what$ is actually the {\em exact} solution
  of a perturbed problem (where the perturbation is
  on the regularizer). This strategy (of showing that the computed solution to a problem is the 
  exact solution of a perturbed problem) is inspired by the {\em backward error analysis} technique
  commonly employed in numerical analysis~\citep{JHW61a}\footnote{J. H. Wilkinson received the Turing Award in 1970, 
  partly for his work on backward error analysis}. 


\section{Experimental Results}
\label{sec:exp}
We conduct several experiments and show that the proposed \atomic 
and \wild have superior performance compared to other state-of-the-art 
parallel coordinate descent algorithms. We consider the hinge loss and five 
datasets: \news, \covtype, \rcv1, \webspam, and \kddb. Detailed information is 
shown in Table \ref{tab:data-stat}. To have a fair comparison, we implement 
all compared methods in C++ using OpenMP as the parallel programming 
framework. All the experiments are performed on an Intel multi-core dual-socket 
machine with 256 GB memory. Each socket is associated with 10 computation 
cores. We explicitly enforce that all the threads use cores from the same 
socket to avoid inter-socket communication. Our codes will be publicly available. 
We focus on solving the (hinge loss) SVM (see (5) in the Appendix) in the experiments,
but the algorithms can also be applied to other objective functions.
Note that some of the figures are in Appendix 6. 

{\bf Serial Baselines.}
\begin{compactitem}
  \item \sdcd: we implement Algorithm \ref{alg:sdcd}. Instead of sampling
    with replacement, a random permutation is used to enforce random sampling 
    without replacement. 
  \item \liblinear: we use the implementation in 
    \url{http://www.csie.ntu.edu.tw/~cjlin/liblinear}. This implementation is 
    equivalent to \sdcd with the shrinking strategy.
\end{compactitem}

{\bf Compared Parallel Implementation.}
\begin{compactitem}
  \item \asdcd: We implement the proposed three variants of Algorithm 
    \ref{alg:asdcd} using \sdcd as the building block: {\it Wild}, {\it Atomic}, 
    and {\it Lock}. 
  \item \cocoa: We implement a multi-core version of \cocoa~\citep{MJ14a} with 
    $\beta_K = 1$ and \sdcd as its local dual method.
  \item \ascd: We follow the description in~\citep{JL14a,JL14b} to implement 
    \ascd with the step length $\gamma=\frac{1}{2}$ and the shuffling period 
    $p=10$ as suggested in~\citep{JL14b}. 
\end{compactitem}

\subsection{Convergence in terms of iterations.}
The primal objective function value is 
used to determine the convergence. Note that we still use $P(\what)$ for 
\wild, although the true primal objective should be 
\eqref{eq:hatw_argmin}. As long as $\what^T\bepsilon$ remains small 
enough, the trend of \eqref{eq:hatw_argmin} and $P(\what)$ are similar.

Figure \ref{fig:conv-rcv1}, \ref{fig:conv-webspam}, \ref{fig:conv-kddb} show the 
convergence results of \wild, \atomic, \cocoa, and \ascd with 10 threads in 
terms of number of iterations.  The horizontal line in grey indicates the 
primal objective function value  
obtained by \liblinear using the default stopping condition. The result
for \liblinear is also included for reference. We have the follow observations 
\begin{compactitem}
\item Convergence of three 
  \asdcd variants are almost identical and very 
    close to the convergence behavior of serial \liblinear on three large 
    sparse datasets (\rcv1, \webspam, and \kddb). 
  \item \wild and \atomic converge significantly faster than \cocoa. 
  \item On \covtype, a more dense dataset, all three algorithms 
    (\wild, \atomic, and \cocoa) have slower convergence. 
\end{compactitem}

\subsection{Efficiency.}

{\bf Timing.}
To have a fair comparison, we include both initialization 
and computation into the timing results. For \sdcd, \asdcd, \cocoa, 
initialization takes one pass of entire data matrix (which is $O(nnz(X))$) to 
compute $\|\bx_i\|$ for each instance.  In the initialization stage, \ascd
requires $O(n\times nnz(X))$ time and $O(n^2)$ space to form and store the 
Hessian matrix $Q$ for \eqref{eq:dual}. Thus, we only have results on 
\news for \ascd as all other datasets are too large for \ascd to fit $Q$ in 
even 256 GB memory.  Note that we also parallelize the initialization part for 
each algorithm in our implementation to have a fair comparison. 

Figures \ref{fig:obj-news20}, \ref{fig:obj-covtype}, \ref{fig:obj-rcv1}, \ref{fig:obj-webspam}, \ref{fig:obj-kddb} show the primal objective values in terms of time and 
Figures \ref{fig:acc-news20}, \ref{fig:acc-covtype}, \ref{fig:acc-rcv1}, \ref{fig:acc-webspam}, \ref{fig:acc-kddb} shows the accuracy in terms of time. Note that the
x-axis for \news, \covtype, and \rcv1 is in log-scale. A horizontal line in 
gray in each figure denotes the objective values/accuracy obtained by 
\liblinear using the default stopping condition. We have the following 
observations:
\begin{compactitem}
\item From Figures \ref{fig:obj-rcv1} and \ref{fig:acc-rcv1}, we can 
  see that \ascd is orders of magnitude slower than other approaches including 
  parallel methods and serial reference (\ascd using 10 cores takes 0.4 seconds to run 10 iterations, while all the other parallel approaches
  takes less than 0.14 seconds, and \liblinear takes less than 0.3 seconds). 
  In fact, \ascd is still slower than 
  other methods even when the initialization time is excluded. This is expected 
  because \ascd is a parallel version of a standard coordinate descent method, which 
  is known to be much slower than \sdcd for \eqref{eq:dual}. Since \ascd runs out of memory
  for all the other larger datasets, we do not show the results in other figures. 
\item In most figures, both \asdcd approaches outperform \cocoa. 
  In Figure 
  \ref{fig:acc-kddb}, \kddb shows better accuracy performance in the 
  early stage which can be explained by the ensemble nature of \cocoa. In the 
  long term, it still converges to the accuracy obtained by \liblinear. 
\item For all datasets, \wild is shown to be slightly faster than \atomic. 
  Given the fact that both methods show similar convergence in terms of 
  iterations, this phenomenon can be explained by the effect of atomic 
  operations. We can observe that more dense the dataset, larger the difference 
  between \wild and \atomic. 
\end{compactitem}
\vspace{-3pt}
\subsection{Speedup}
\vspace{-3pt}
We are interested in the following evaluation criterion:
\[
  \text{speedup} := \frac{\text{time taken by the target method with $p$ threads}}
  {\text{time taken by the best serial reference method}},
\]
This criterion is different from {\em scaling}, where the denominator is 
replaced by ``time taken for the target method with single thread.'' Note that 
a method can have perfect scaling but very poor speedup. 
Figures \ref{fig:speedup-news20}, \ref{fig:speedup-covtype}, \ref{fig:speedup-rcv1}, \ref{fig:speedup-webspam}, \ref{fig:speedup-kddb} shows the speedup results, where 1) \sdcd is used as the best serial reference; 2)  
the shrinking heuristic is turned off for all \asdcd and \sdcd to have fair 
comparison; 3) the initialization time is excluded from the computation of 
speedup.  

\begin{compactitem}
\item \wild has very good speedup performance compared to other approaches. It 
  achieves about 6 to 8 speedups using 10 threads on all the datasets. 
\item From Figure \ref{fig:speedup-news20}, we can see that \ascd does not have 
  any ``speedup'' over the serial reference, although it is shown to have 
  almost linear scaling~\citep{JL14b,JL14a}. 
\end{compactitem}

\begin{figure*}[h]
  \centering
  \hspace{-3.3em}
  \begin{minipage}{0.38\linewidth}
    \begin{tabular}{c}
      \vspace{-0.8em}
      \subfigure[Convergence]{
        \includegraphics[width=\linewidth]{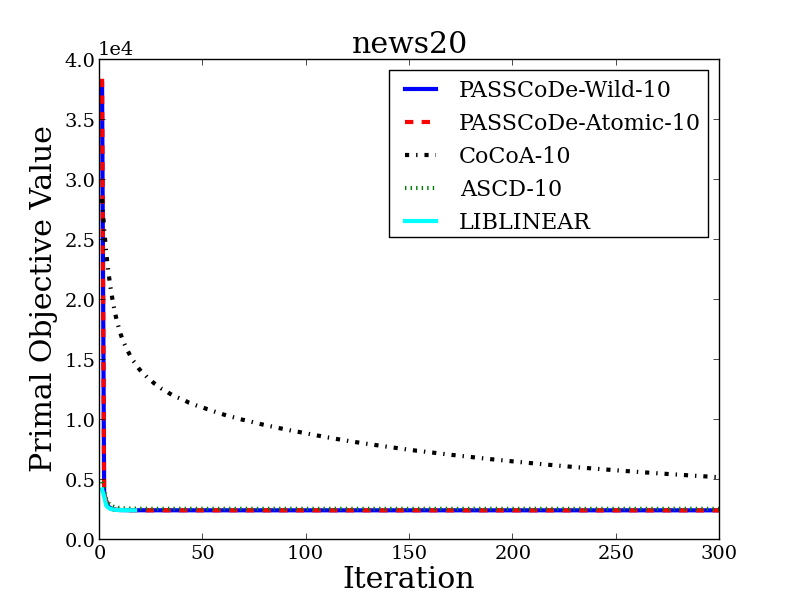}
        \label{fig:conv-news20}
      } \\
      \vspace{-0.8em}
      \subfigure[Objective]{
        \includegraphics[width=\linewidth]{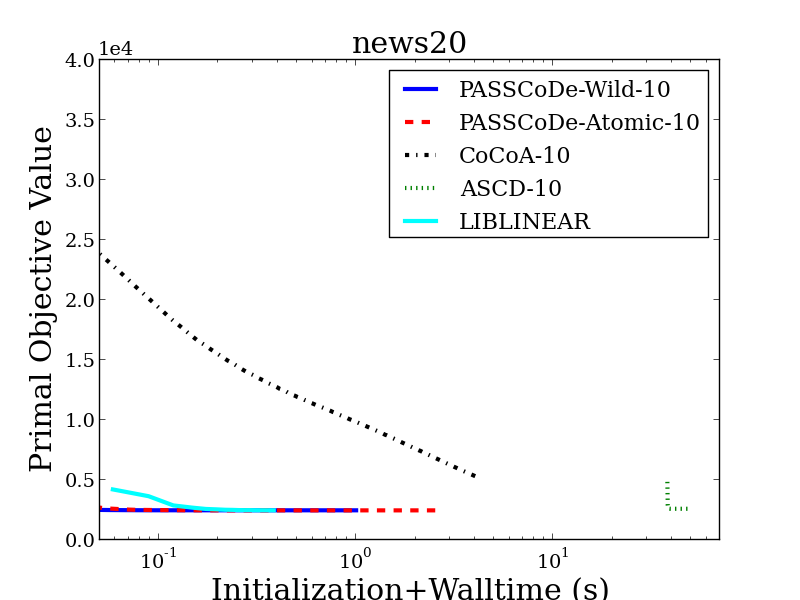}
        \label{fig:obj-news20}
      } \\
      \vspace{-0.8em}
      \subfigure[Accuracy]{
        \includegraphics[width=\linewidth]{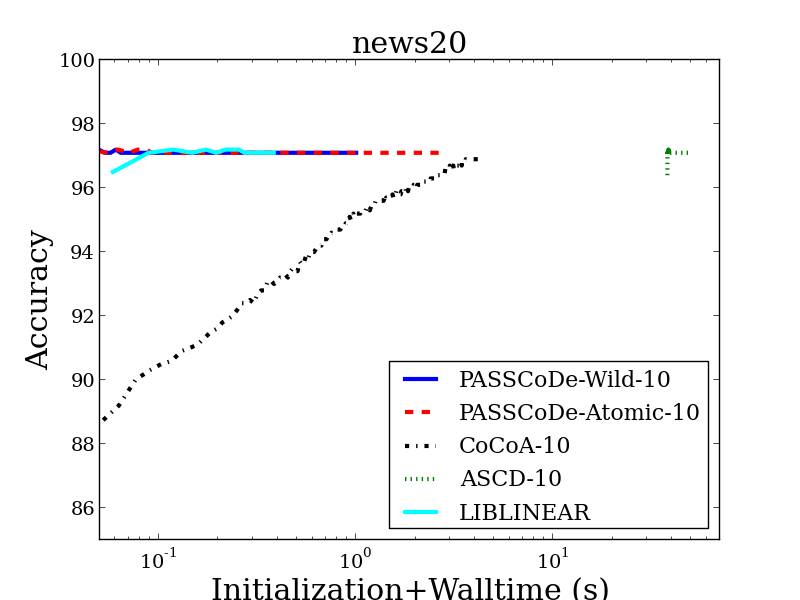}
        \label{fig:acc-news20}
      } \\
      \vspace{-0.8em}
      \subfigure[Speedup]{
        \includegraphics[width=\linewidth]{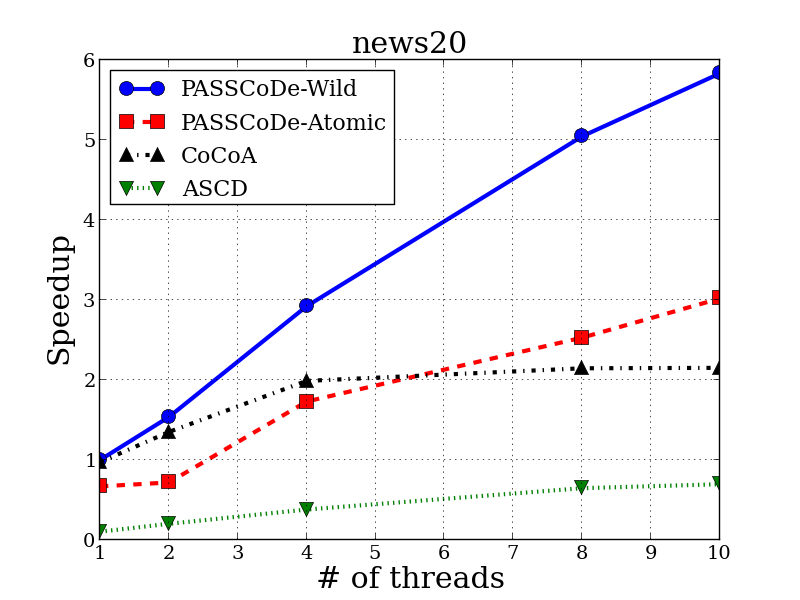}
        \label{fig:speedup-news20}
      } 
    \end{tabular}
    \caption{\news dataset}
    \label{fig:conv}
  \end{minipage}
  \hspace{-2em}
  \begin{minipage}{0.38\linewidth}
    \begin{tabular}{c}
      \vspace{-0.8em}
      \subfigure[Convergence]{
        \includegraphics[width=\linewidth]{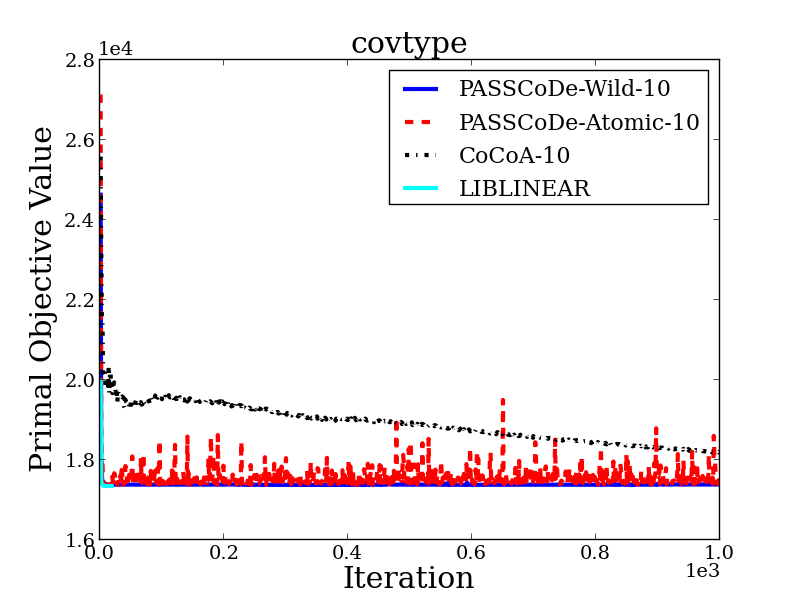}
        \label{fig:conv-covtype}
      } \\
      \vspace{-0.8em}
      \subfigure[Objective]{
        \includegraphics[width=\linewidth]{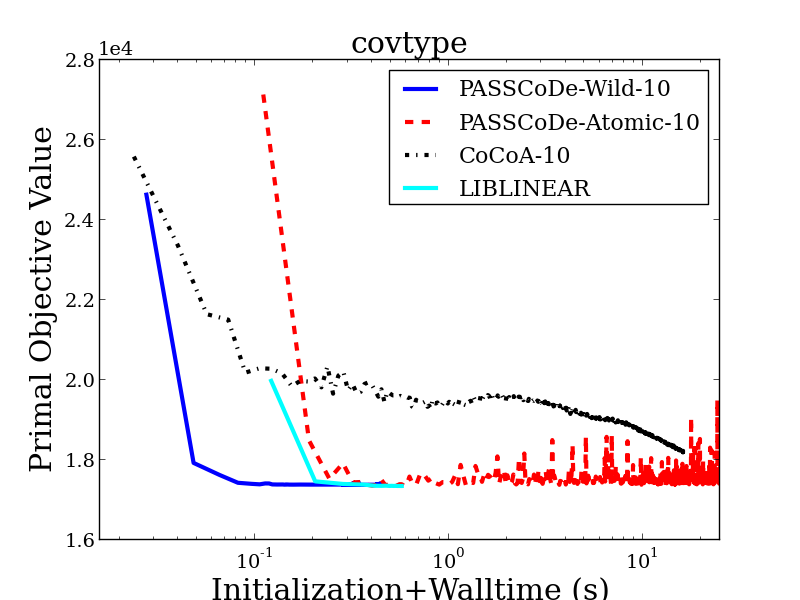}
        \label{fig:obj-covtype}
      } \\
      \vspace{-0.8em}
      \subfigure[Accuracy]{
        \includegraphics[width=\linewidth]{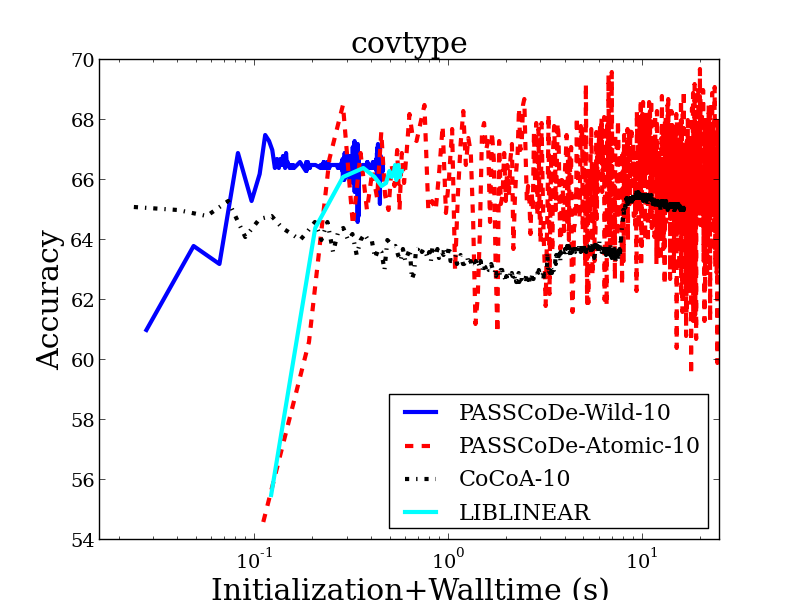}
        \label{fig:acc-covtype}
      } \\
      \vspace{-0.8em}
      \subfigure[Speedup]{
        \includegraphics[width=\linewidth]{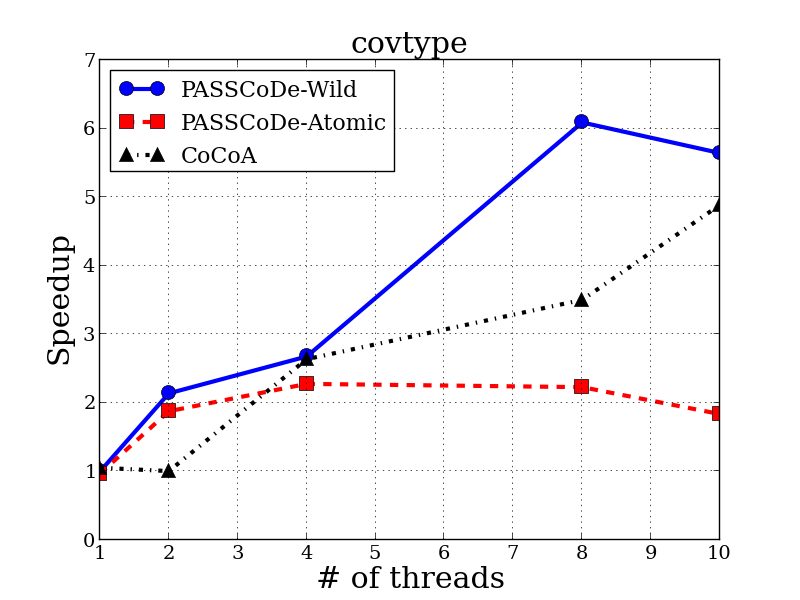}
        \label{fig:speedup-covtype}
      } 
    \end{tabular}
    \caption{\covtype dataset}
    \label{fig:obj}
  \end{minipage}
  \hspace{-2em}
  \begin{minipage}{0.38\linewidth}
    \begin{tabular}{c}
      \vspace{-0.8em}
      \subfigure[Convergence]{
        \includegraphics[width=\linewidth]{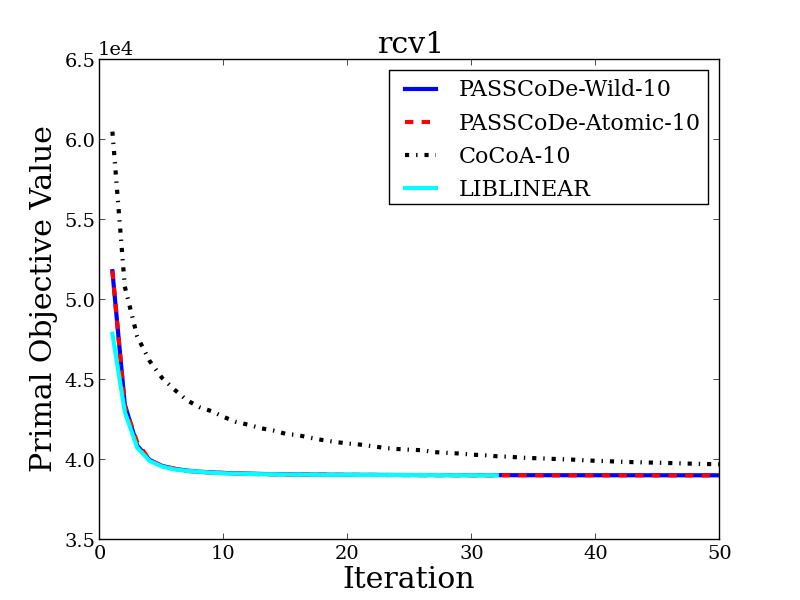}
        \label{fig:conv-rcv1}
      } \\
      \vspace{-0.8em}
      \subfigure[Objective]{
        \includegraphics[width=\linewidth]{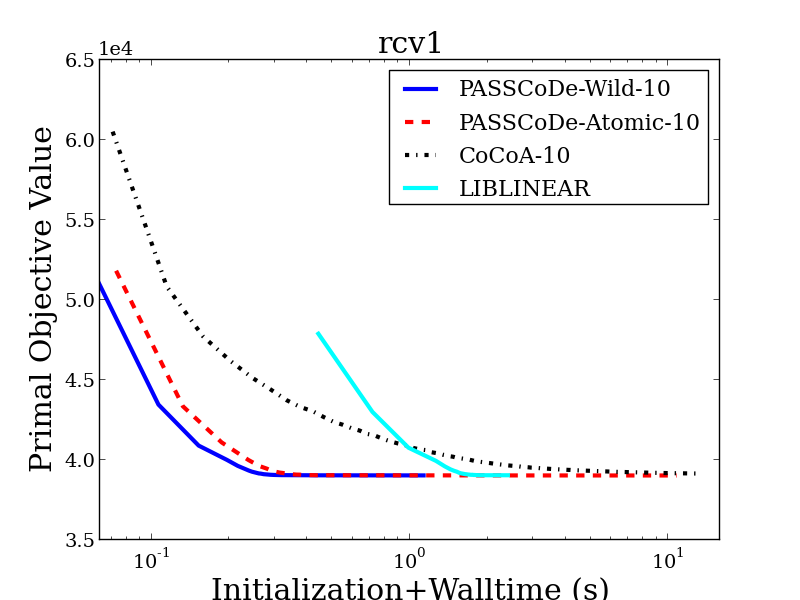}
        \label{fig:obj-rcv1}
      } \\
      \vspace{-0.8em}
      \subfigure[Accuracy]{
        \includegraphics[width=\linewidth]{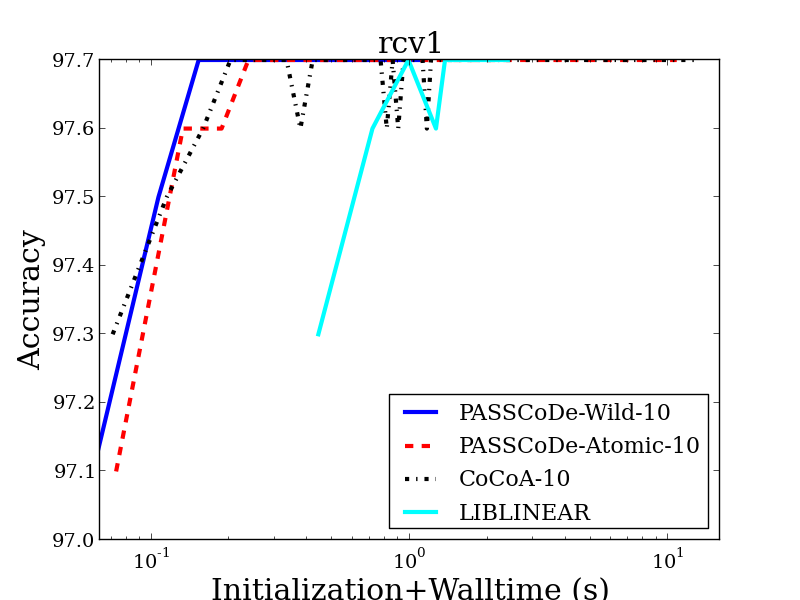}
        \label{fig:acc-rcv1}
      } \\
      \vspace{-0.8em}
      \subfigure[Speedup]{
        \includegraphics[width=\linewidth]{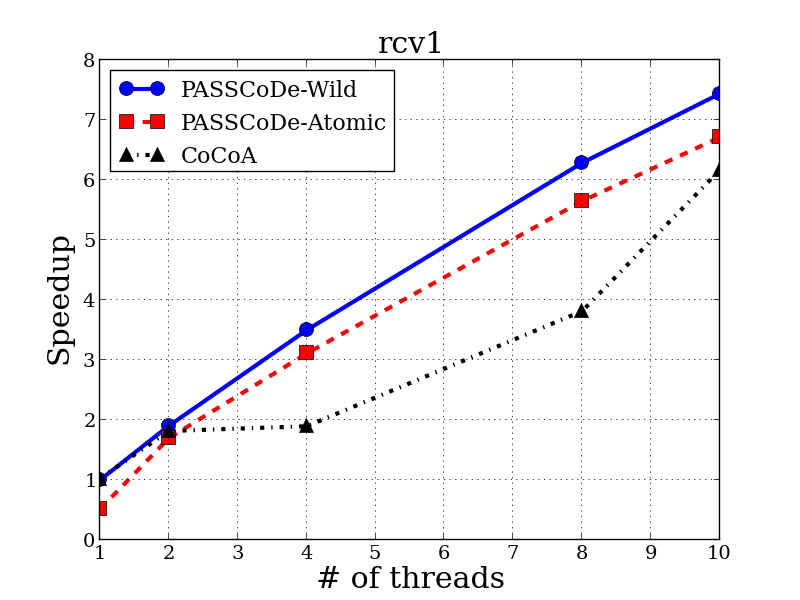}
        \label{fig:speedup-rcv1}
      } 
    \end{tabular}
    \caption{\rcv1 dataset}
    \label{fig:conv}
  \end{minipage}
\end{figure*}

\begin{figure*}[h]
  \centering
  \begin{minipage}{0.38\linewidth}
    \begin{tabular}{c}
      \vspace{-0.8em}
      \subfigure[Convergence]{
        \includegraphics[width=\linewidth]{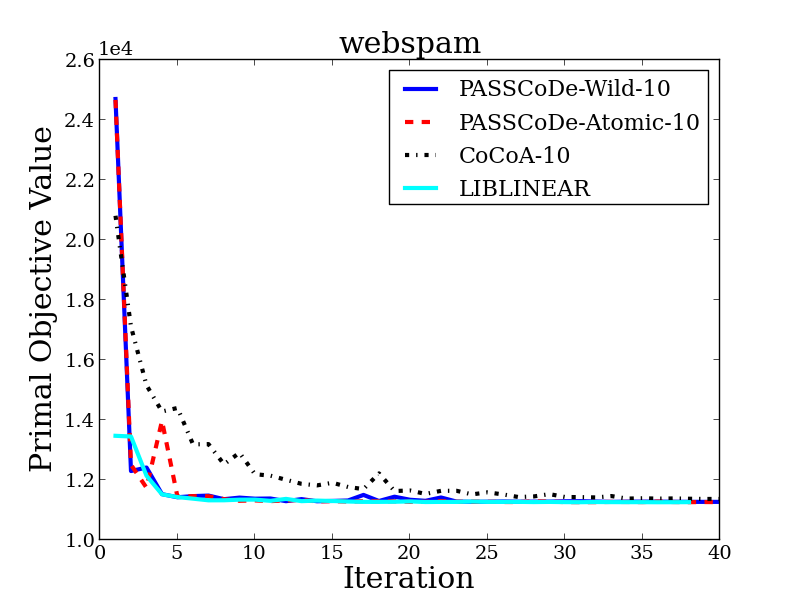}
        \label{fig:conv-webspam}
      } \\
      \vspace{-0.8em}
      \subfigure[Objective]{
        \includegraphics[width=\linewidth]{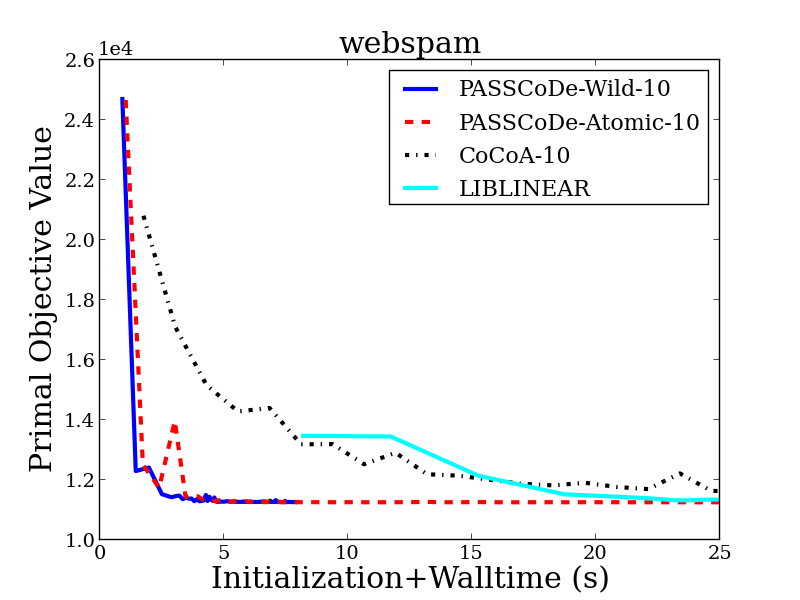}
        \label{fig:obj-webspam}
      } \\
      \vspace{-0.8em}
      \subfigure[Accuracy]{
        \includegraphics[width=\linewidth]{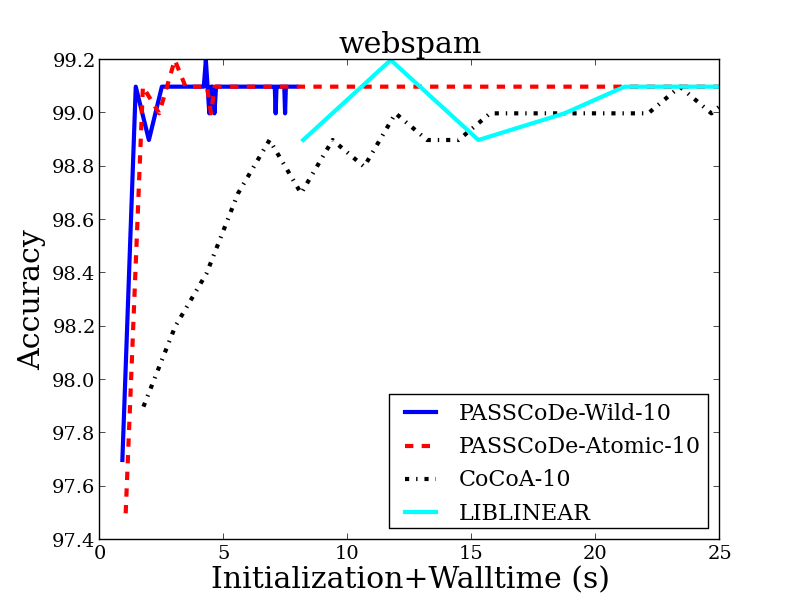}
        \label{fig:acc-webspam}
      } \\
      \vspace{-0.8em}
      \subfigure[Speedup]{
        \includegraphics[width=\linewidth]{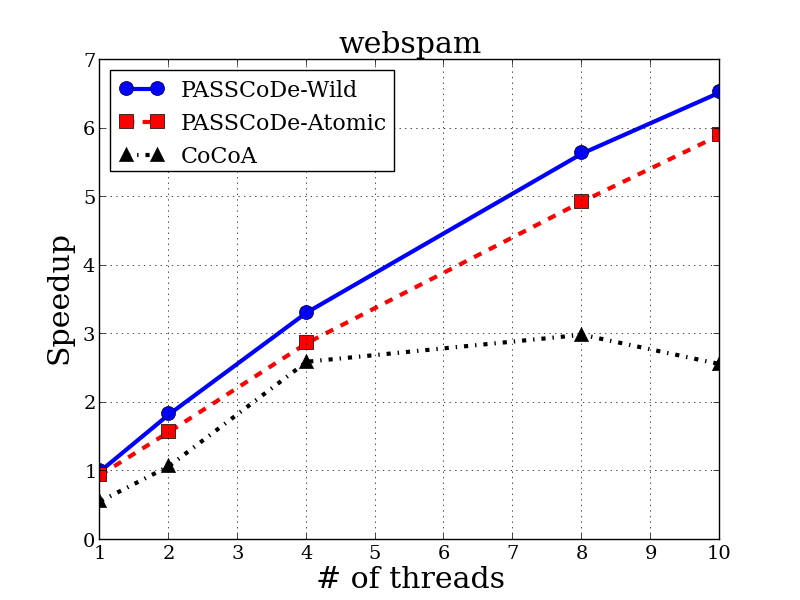}
        \label{fig:speedup-webspam}
      } 
    \end{tabular}
    \caption{\webspam dataset}
    \label{fig:obj}
  \end{minipage}
  \begin{minipage}{0.38\linewidth}
    \begin{tabular}{c}
      \vspace{-0.8em}
      \subfigure[Convergence]{
        \includegraphics[width=\linewidth]{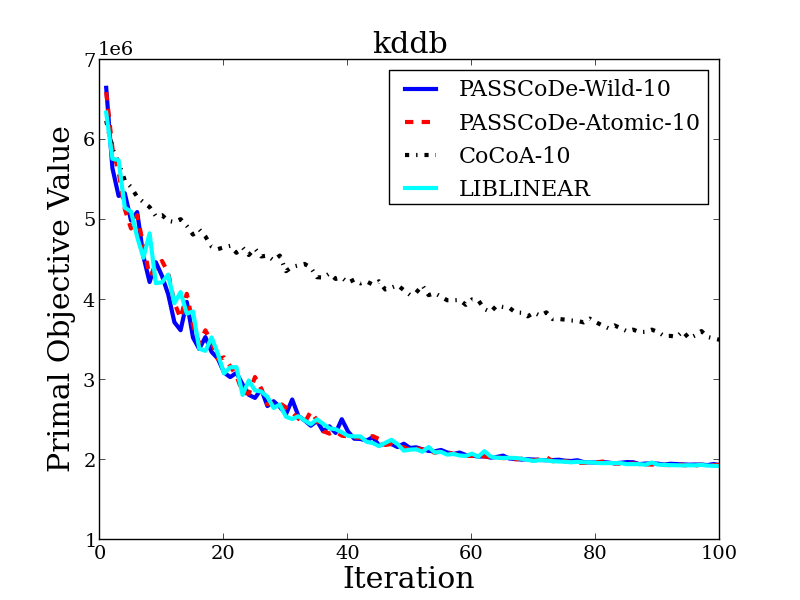}
        \label{fig:conv-kddb}
      } \\
      \vspace{-0.8em}
      \subfigure[Objective]{
        \includegraphics[width=\linewidth]{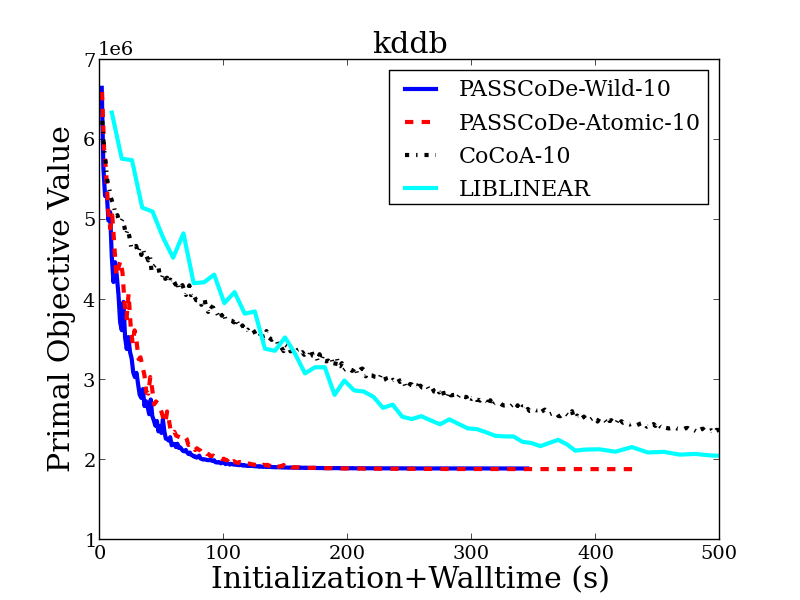}
        \label{fig:obj-kddb}
      } \\
      \vspace{-0.8em}
      \subfigure[Accuracy]{
        \includegraphics[width=\linewidth]{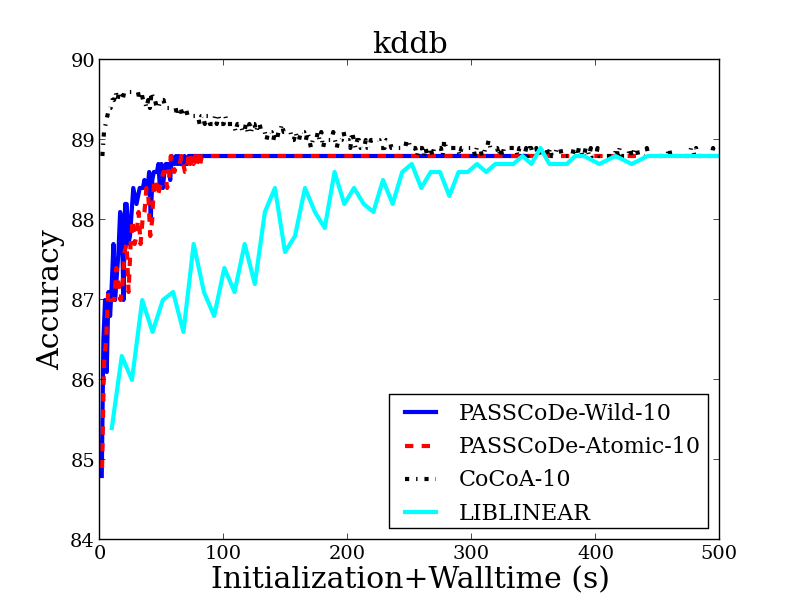}
        \label{fig:acc-kddb}
      }\\ 
      \vspace{-0.8em}
      \subfigure[Speedup]{
        \includegraphics[width=\linewidth]{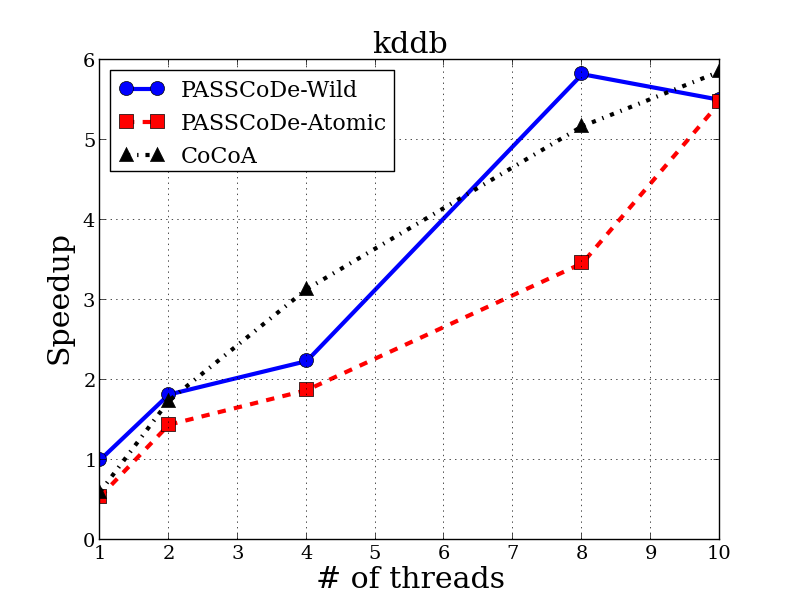}
        \label{fig:speedup-kddb}
      } 
    \end{tabular}
    \caption{\kddb dataset}
    \label{fig:acc}
  \end{minipage}
\end{figure*}

\section{Conclusions}

In this paper, we present a family of parallel asynchronous stochastic dual coordinate descent algorithms in the
shared memory multi-core setting, where each thread repeatedly selects a random dual variable and conducts
coordinate updates using the primal variables that are stored in the shared memory. We analyze
the convergence properties when different locking/atomic mechanism is used. For the setting
with atomic updates, we show the linear convergence under certain condition. 
For the setting without any lock or atomic write, which achieves the best speed up, we present a backward error analysis to show
that the primal variables obtained by the algorithm is the exact solution for a primal problem with perturbed regularizer. 
Experimental results show that our algorithms are much faster than previous parallel coordinate descent solvers. 


\clearpage

\bibliography{sdp}
\bibliographystyle{icml2015}
\clearpage
\appendix 
\section{Linear Convergence for \atomic }
\subsection{Notations and Prepositions}
\subsubsection{Notations}
\begin{itemize}
  \item For all $i=1,\ldots,n$, we have the following definitions:
    \begin{align*}
      h_i(u) :=& \frac{\ell^*_i(-u)}{\|\bx_i\|^2} \\
    \prox_i(s) :=& \arg\min_{u}\ \frac{1}{2}(u-s)^2 + h_i(u) \\
      T_i(\bw,s) :=& \arg\min_{u}\ \frac{1}{2} \|\bw + (u-s)\bx_i\|^2 + \ell^*_i(-u)\\
      =& \arg\min_{u}\ \frac{1}{2} \left[u - (s - \frac{\bw^T\bx_i}{\|\bx_i\|^2})\right]^2 + h_i(u), \\
    \end{align*}
    where $\bw \in R^{d}$ and $s\in R$. 
    We also denote $\prox(\bs)$ as the proximal operator from $R^n$ to $R^n$ such that 
    $(\prox(\bx))_i = \prox_i(s_i)$.   We can see the connection of the above 
    operator and the proximal operator: 
    $T_i(\bw,s) = \prox_{i}(s - \frac{\bw^T\bx_i}{\|\bx_i\|^2})$.
  \item Let $\{\AL^j\}$ and $\{\hat\bw^j\}$ be the sequence generated/maintained 
    by Algorithm \ref{alg:asdcd} using  
    \begin{align*}
      \alpha_t^{j+1} &= 
      \begin{cases}
        T_t(\hat\bw^j, \alpha_t^j) & \text{ if } t = i(j), \\
                       \alpha_t^j  & \text{ if } t \neq i(j),
      \end{cases}
    \end{align*}
    where $i(j)$ is the index selected at $j$-th iteration.  For convenience, 
    we define 
    \[
    \Delta\alpha_j = \alpha^{j+1}_{i(j)} - \alpha^j_{i(j)}.
    \]
  \item Let $\{\ALt^j\}$ be the sequence defined by 
    \begin{align*}
      \alphat_t^{j+1} &= T_t(\hat\bw^j, \alpha_t^j)\quad \forall t=1,\ldots,n.
    \end{align*}
    Note that $\alphat_{i(j)}^{j+1} = \alpha_{i(t)}^{j+1}$ and 
    $\ALt^{j+1} = \prox(\AL^j - \bar X \what^j)$.   
  \item Let $\bar{\bw}^j = \sum_i \alpha^j_i \bx_i$ be the ``true'' primal  
    variables corresponding to $\AL^j$.  
\end{itemize}
\subsubsection{Prepositions}
\begin{prep}
  \label{prep:exp_normsq}
  \begin{equation}
    E_{i(j)}(\|\AL^{j+1} - \AL^{j}\|^2) = \frac{1}{n} 
    \|\ALt^{j+1}-\AL^{j}\|^2. 
    \label{eq:exp_normsq}
  \end{equation}
\end{prep}
\begin{proof}
It can be proved by the definition of $\ALt$ and the assumption that $i(j)$ is 
uniformly random selected from $\{1,\ldots,n\}$.  
\end{proof}
\begin{prep}
  \begin{align}
    \|\bar X \bar\bw^j - \bar X \what^j\| \le M \sum^{j-1}_{t=j-\tau} |\Delta \alpha_t|.
    \label{eq:stale}
  \end{align}
  \label{prep:Xw_diff}
\end{prep}
\begin{proof}
  \begin{align*}
  \|\bar{X}\bar\bw^j - \bar{X}\what^j\| &= \|\bar{X}(\sum_{(t,k)\in \Z^j\setminus \U^j} (\Delta\alpha_t) X_{i(t),k}\be_k)\| 
 = \|\sum_{(t,k)\in \Z^j \setminus \U^j} (\Delta\alpha_t)\bar{X}_{:,k}X_{i(t),k}\|\\
  &\leq \sum_{t=j-1}^{j-\tau} |\Delta\alpha_t| M_i \leq M\sum^{j-1}_{t=j-\tau} |\Delta\alpha_t|
\end{align*}
\end{proof}
\begin{prep}
  For any $\bw_1,\bw_2 \in R^d$ and $s_1, s_2 \in R$,
  \label{prep:nonexp}
    \begin{align}
      |T_i(\bw_1, s_1) - T_i(\bw_2, s_2)| &\le |s_1 - s_2 + 
      \frac{(\bw_1-\bw_2)^T\bx_i}{\|\bx_i\|^2}|.
      \label{eq:nonexp}
    \end{align}
\end{prep}
\begin{proof}
It can be proved by the connection of $T_i(\bw,s)$ and $\prox_i(\cdot)$ and the non-expansiveness of the 
  proximal operator.
\end{proof}

\begin{prep}
  \label{prep:lagcond}
  Let $M\ge1$, $q=\frac{6(\tau+1)eM}{\sqrt{n}}$, $\rho = (1+q)^2$, and $\theta=\sum_{t=1}^\tau \rho^{t/2}$.  
  If $M\ge 1$ and $q(\tau+1)\le 1$, then $\rho^{(\tau+1)/2} \le e$, and
  \begin{align}
    \rho^{-1} \le 1 - \frac{4+4M+4M\theta}{\sqrt{n}}.
    \label{eq:lagcond}
  \end{align}
\end{prep}
\begin{proof}
  By the definition of $\rho$ and the condition $q(\tau+1)\le 1$, we have
  \[
  \rho^{(\tau+1)/2} = \left(\left(\rho^{1/2}\right)^{1/q}\right)^{q(\tau+1)} = 
  \left((1+q)^{1/q}\right)^{q(\tau+1)} \le e^{q(\tau+1)} \le e.
  \]
  By the definitions of $q$, we know that 
  \[
  q = \rho^{1/2} - 1 = \frac{6(\tau+1)eM}{\sqrt{n}} \Rightarrow \frac{3}{2} =  
  \frac{\sqrt{n}(\rho^{1/2}-1)}{4(\tau+1)eM}. 
  \]
  We can derive 
  \begin{align*}
    \frac{3}{2} &= \frac{\sqrt{n}(\rho^{1/2}-1)}{4(\tau+1)eM}\\
                &\le \frac{\sqrt{n}(\rho^{1/2}-1)}{4(\tau+1)\rho^{(\tau+1)/2}M} 
                \quad\quad\quad \because \rho^{(\tau+1)/2}\le e \\
                &\le \frac{\sqrt{n}(\rho^{1/2}-1)}{4(1+\theta)\rho^{1/2}M} 
                \quad\quad\quad\quad\quad \because 1+\theta=\sum_{t=0}^\tau \rho^{t/2} \le (\tau+1)\rho^{\tau/2} \\
                &= \frac{\sqrt{n}(1-\rho^{-1/2})}{4(1+\theta)M} \\
                &\le \frac{\sqrt{n}(1-\rho^{-1})}{4(1+\theta)M}
                \quad\quad\quad\quad\quad\quad \because \rho^{-1/2} \le 1
  \end{align*}
  Combining the condition that $M\ge 1$ and $1+\theta\ge 1$, we have
  \begin{align*}
    \frac{\sqrt{n}(1-\rho^{-1}) - 4}{4(1+\theta)M} \ge 
    \frac{\sqrt{n}(1-\rho^{-1})}{4(1+\theta)M} - \frac{1}{2}\ge 1,
  \end{align*}
  which leads to 
  \begin{align*}
    4(1+\theta)M &\le \sqrt{n} - \sqrt{n}\rho^{-1} - 4 \\
    \rho^{-1} &\le 1 - \frac{4+4M+4M\theta}{\sqrt{n}}.
  \end{align*}
\end{proof}

\subsection{Proof of Lemma \ref{lm:rho_decrease}}
\label{app:lemma}
Similar to \cite{JL14a}, we prove Eq. \eqref{eq:rho_decrease} by induction. First,
we know that for any two vectors $\ba$ and $\bb$, we have 
\begin{align*}
  \|\ba\|^2 - \|\bb\|^2 \le 2 \|\ba\|\|\bb-\ba\|. 
\end{align*}
See \cite{JL14a} for a proof for the above inequality. Thus, for all $j$ , we have 
\begin{equation}
  \|\AL^{j-1}-\ALt^j\|^2 - \|\AL^j - \ALt^{j+1}\|^2 
  \leq 2 \|\AL^{j-1} - \ALt^j\| \|\AL^{j} - \ALt^{j+1} - \AL^{j-1} + \ALt^{j}\|. \label{eq:NN} 
\end{equation}
The second factor in the r.h.s of \eqref{eq:NN} is bounded as follows:
\begin{align}
  &\|\AL^j - \ALt^{j+1} - \AL^{j-1}+\ALt^j\| \nonumber \\
  &\leq \|\AL^j - \AL^{j-1}\| + \|\prox(\AL^j-\bar{X}\what^j) - \prox(\AL^{j-1}-\bar{X}\what^{j-1})\| \nonumber\\
  &\leq \|\AL^j - \AL^{j-1}\| + \|(\AL^j-\bar{X}\what^j) - (\AL^{j-1}-\bar{X}\what^{j-1})\| \nonumber\\
  &\leq \|\AL^j - \AL^{j-1}\| + \|\AL^j - \AL^{j-1}\| + \|\bar{X}\what^j - \bar{X} \what^{j-1}\| \nonumber \\
  &= 2 \|\AL^j - \AL^{j-1}\| + \|\bar{X}\what^j - \bar{X} \what^{j-1}\| \nonumber \\
  &= 2\|\AL^j - \AL^{j-1}\| + \|\bar{X}\what^j - \bar{X}\wbar^j + \bar{X}\wbar^j - \bar{X}\wbar^{j-1} + \bar{X}\wbar^{j-1}-\bar{X}\what^{j-1}\| \nonumber\\
  &\leq 2 \|\AL^j - \AL^{j-1}\| + \|\bar{X}\wbar^j - \bar{X}\wbar^{j-1}\| + \|\bar{X}\what^j - \bar{X}\wbar^j\| + \|\bar{X}\wbar^{j-1}-\bar{X}\what^{j-1}\| \nonumber\\
  &\leq (2+M)\|\AL^j - \AL^{j-1}\| + \sum_{t=j-\tau}^{j-1}\|\Deltaalpha_t \| M + \sum_{t=j-\tau-1}^{j-2} \|\Deltaalpha_t\| M \nonumber\\
  &= (2+2M) \|\AL^j - \AL^{j-1}\| + 2M\sum_{t=j-\tau-1}^{j-2} \|\Deltaalpha_t\| \label{eq:MM}
\end{align}
Now we prove \eqref{eq:rho_decrease} by induction. 

{\bf Induction Hypothesis. }
Due to Preposition \ref{prep:exp_normsq}, we prove the following equivalent 
statement. For all $j$, 
\begin{equation}
  E(\|\AL^{j-1}-\ALt^{j}\|^2) \leq \rho E(\|\AL^j - \ALt^{j+1}\|^2), 
  \label{eq:rho_decrease_2}
\end{equation}

{\bf Induction Basis.} When $j=1$, 
\begin{equation*}
  \|\AL^1 - \ALt^2 + \AL^0 - \ALt^1\| \leq (2+2M) \|\AL^1 - \AL^{0}\|. 
\end{equation*}
By taking the expectation on \eqref{eq:NN}, we have
\begin{align*}
  E[\|\AL^0-\ALt^1\|^2] - E[\|\AL^1 - \ALt^2\|^2]  &\leq 2 E[\|\AL^0 - \ALt^1\|\|\AL^1 - \ALt^2 - \AL^0 + \ALt^1\|] \\
  & \leq (4+4M) E(\|\AL^0 - \ALt^1\|\|\AL^0 - \AL^1\|). 
\end{align*}
From \eqref{eq:exp_normsq}  we have $E[\|\AL^0 - \AL^1\|^2]= \frac{1}{n}\|\AL^0 - \ALt^1\|^2$. 
Also, by AM-GM inequality, for any $\mu_1, \mu_2>0$ and any $c>0$, we have 
\begin{equation}
  \mu_1\mu_2 \leq \frac{1}{2} (c \mu_1^2 + c^{-1}\mu_2^2). 
  \label{eq:mu1mu2}
\end{equation}
Therefore, we have
\begin{align*}
  &E[\|\AL^0 - \ALt^1\|\|\AL^0 - \AL^1\|] \\
  &\leq \frac{1}{2} E\Bigl[  n^{1/2} \|\AL^0 - \AL^1\|^2 + n^{-1/2} \|\ALt^1 - \AL^0\|^2 \Bigl] \\
  &= \frac{1}{2} E\Bigl[ n^{-1/2} \|\AL^0 - \ALt^1\|^2 + n^{-1/2} \|\ALt^1 - 
  \AL^0\|^2  \Bigl] \quad\quad\quad \text{by \eqref{eq:exp_normsq}}\\
  &= {n}^{-1/2}E[\|\AL^0 - \ALt^1\|^2]. 
\end{align*}
Therefore, 
\begin{equation*}
  E[\|\AL^0 - \ALt^1\|^2]  - E[\|\AL^1 - \ALt^2\|^2] \leq
  \frac{4+4M}{\sqrt{n}} E[\|\AL^0 - \ALt^1\|^2], 
\end{equation*}
which implies
\begin{equation}
  E[\|\AL^0 - \AL^1\|^2] \leq \frac{1}{1-\frac{4+4M}{\sqrt{n}}} E[\|\AL^1 - \ALt^2\|^2]
  \le \rho E[\|\AL^1 - \ALt^2\|^2],
  \label{eq:gg58}
\end{equation}
where the last inequality is based on Preposition \ref{prep:lagcond} and the 
fact $\theta M \ge 1$.    

{\bf Induction Step. } By the induction hypothesis, we assume 
\begin{equation}
  E[\|\AL^{t-1}-\ALt^{t}\|^2] \leq \rho E[\|\AL^t - \ALt^{t+1}\|^2] \quad \forall t\leq j-1. 
  \label{eq:ind}
\end{equation}
The goal is to show 
\begin{equation*}
  E[\|\AL^{j-1}-\ALt^j\|^2] \leq \rho E[\|\AL^{j}-\ALt^{j+1}\|^2]. 
\end{equation*}
First, we show that for all $t < j$, 
\begin{equation}
E\Bigl[ \|\AL^t - \AL^{t+1}\|\|\AL^{j-1} - \ALt^{j}\|  \Big]  \le 
\frac{\rho^{(j-1-t)/2}}{\sqrt{n}} E\Bigl[ \|\AL^{j-1}- \ALt^{j} \|^2\Bigl]
\label{eq:GG}
\end{equation}
\begin{proof}
  By \eqref{eq:mu1mu2} with $c=n^{1/2}\beta$, where $\beta=\rho^{(t+1-j)/2}$, 
\begin{align*}
  &E\Bigl[ \|\AL^t - \AL^{t+1}\|\|\AL^{j-1} - \ALt^{j}\|  \Big] \\
  &\leq \frac{1}{2} E\Bigl[ n^{1/2} \beta \|\AL^t - \AL^{t+1}\|^2 + n^{-1/2} \beta^{-1} \|\AL^{j-1} - \ALt^{j}\|^2 \Bigl] \\
  &=\frac{1}{2} E\Bigl[ n^{1/2} \beta E[\|\AL^t - \AL^{t+1}\|^2] + n^{-1/2} \beta^{-1} \|\AL^{j-1} - \ALt^{j}\|^2 \Bigl] \\
  &=\frac{1}{2} E\Bigl[ n^{-1/2} \beta \|\AL^t - \ALt^{t+1}\|^2 + n^{-1/2} \beta^{-1} \|\AL^{j-1} - \ALt^{j}\|^2 \Bigl] 
  \quad\quad \text{by Preposition \ref{prep:exp_normsq}} \\
  &\le \frac{1}{2} E\Bigl[ n^{-1/2} \beta \rho^{j-1-t}\|\AL^{j-1} - \ALt^{j}\|^2 + n^{-1/2} \beta^{-1} \|\AL^{j-1} - \ALt^{j}\|^2 \Bigl] 
  \quad\quad \text{by Eq. \eqref{eq:ind}} \\
  &\le \frac{1}{2} E\Bigl[ n^{-1/2} \beta^{-1}\|\AL^{j-1} - \ALt^{j}\|^2 + n^{-1/2} \beta^{-1} \|\ALt^{j-1} - \AL^{j}\|^2 \Bigl] 
  \quad\quad \text{by the definition of $\beta$} \\
    &\le \frac{\rho^{(j-1-t)/2}}{\sqrt{n}} E\Bigl[ \|\AL^{j-1} - \ALt^{j}\|^2\Bigl]
\end{align*}
\end{proof}
Let $\theta=\sum_{t=1}^{\tau} \rho^{t/2}$. 
We have
\begin{align*}
  &E[\|\AL^{j-1}-\ALt^j\|^2] - E[ \|\AL^{j}-\ALt^{j+1}\|^2 ] \\ 
  &\leq E\Bigl[ 2\|\AL^{j-1}-\ALt^j\|\bigl( (2+2M)\|\AL^j-\AL^{j-1}\|+2M\sum_{t=j-\tau-1}^{j-1}\|\AL^t - \AL^{t-1}\|  \bigl)  \Bigl] 
  \quad\text{by \eqref{eq:NN}, \eqref{eq:MM}}
  \\
  &= (4+4M)E(\|\AL^{j-1}-\ALt^j\|\|\AL^j - \AL^{j-1}\|) + 4M\sum_{t=j-\tau-1}^{j-1} E\Bigl[ \|\AL^{j-1}-\ALt^j\|\|\AL^t - \AL^{t-1}\|  \Bigl] \\
  &\leq (4+4M)n^{-1/2}E[\|\ALt^j - \AL^{j-1}\|^2] + 4M n^{-1/2} E[\|\AL^{j-1}-\ALt^j\|^2]\sum_{t=j-1-\tau}^{j-2} \rho^{(j-1-t)/2} 
  \quad\text{by \eqref{eq:GG}}\\
  &\leq (4+4M)n^{-1/2}E[\|\ALt^j - \AL^{j-1}\|^2] + 4Mn^{-1/2}\theta E[\|\AL^{j-1}-\ALt^j\|^2]\\
  &\leq \frac{4+4M+4M\theta}{\sqrt{n}} E[\|\AL^{j-1}-\ALt^j\|^2], 
\end{align*}
which implies that 
\begin{align*}
  E[\|\AL^{j-1}-\ALt^{j}\|^2] \le \frac{1}{1 - \frac{4+4M+4M\theta}{\sqrt{n}}} 
  E[\|\AL^{j}-\ALt^{j+1}\|^2] \le \rho E[\|\AL^{j}-\ALt^{j+1}\|^2],
\end{align*}
where the last inequality is based on Preposition \ref{prep:lagcond}. 

\subsection{Proof of Theorem \ref{thm:converge_atomic}}
\label{app:converge_atomic}

First, we define $T(\bw, \AL)$ to be a $n$-dimensional vector such that
\begin{equation*}
  (T(\bw, \AL))_t = T_t(\bw, \AL_t) \text{ for all } t, 
\end{equation*}
We can then bound the distance $E[\|T(\bw^j, \AL^j)-T(\hat{\bw}^j, \AL^j)\|^2]$ by (we omit the expectation in the following derivation):
\begin{align*}
  \|T(\bw^j, \AL^j)-T(\hat{\bw}^j, \AL^j)\|^2 &= \sum_{t=1}^n \big(T_t(\bw^j, \alpha_t^j)-T_t(\hat{\bw}^j, \alpha_t^j)\big)^2 \\
  &\leq \sum_t \big( \frac{(\bw^j - \hat{\bw}^j)^T \bx_t}{\|\bx_t\|^2}  \big)^2 \ \ \text{ (By Proposition \ref{prep:Xw_diff})}\\
  &= \|\bar{X}(\bw^j - \hat{\bw}^j)\|^2 \\
  &\leq M^2 \big( \sum_{t=j-\tau}^{j-1}\|\AL^{t+1}-\AL^t\| \big)^2 \ \ \text{ (By Proposition \ref{prep:nonexp})} \\
  &\leq \tau M^2 \big( \sum_{t=j-\tau}^{j-1} \|\AL^{t+1}-\AL^t\|^2 \big) \\
  &\leq \tau M^2 \big( \sum_{t=1}^\tau \rho^t \|\AL^j-\AL^{j+1}\|^2 \big) \ \ \text{ (By Lemma \ref{lm:rho_decrease})}\\
  &\leq \frac{\tau M^2}{n} (\sum_{t=1}^\tau \rho^t) \|T(\hat{\bw}^j, \AL^j)-\AL^j\|^2\\
  &\leq \frac{\tau^2 M^2}{n} \rho^\tau \|T(\hat{\bw}^j, \AL^j)-\AL^j\|^2 
\end{align*}
Since $\rho^{(\tau+1)/2}\leq e$, we have 
  $\rho^{\tau+1}\leq e^2 $, 
  so $\rho^{\tau}\leq e^2$ since $\rho\geq 1$. 
  Therefore, 
  \begin{equation}
    \|T(\bw^j, \AL^j)-T(\hat{\bw}^j, \AL^j)\|^2 \leq \frac{\tau^2 M^2 e^2}{n}
    \|T(\hat{\bw}^j, \AL^j)-\AL^j\|^2. 
    \label{eq:Tdiff_bound}
  \end{equation}
  As a result, 
  \begin{align}
    \|T(\bw^j, \AL^j)-\AL^j\|^2 &= \|T(\bw^j, \AL^j)-T(\hat{\bw}^j, \AL^j) + T(\hat{\bw}^j, \AL^j) - \AL^j\|^2\nonumber\\
    &\leq 2\big(\|T(\bw^j, \AL^j)-T(\hat{\bw}^j, \AL^j)\|^2 + \|T(\hat{\bw}^j, \AL^j)-\AL^j\|^2\big) \nonumber\\
    &\leq 2(1+\frac{e^2\tau^2 M^2}{n}) \|T(\hat{\bw}^j, \AL^j)-\AL^j\|^2. \label{eq:T_alpha_bound}
  \end{align}

Next, we bound the decrease of objective function value by 
\begin{align*}
  D(\AL^j) - D(\AL^{j+1}) &= D(\AL^j) - D(\bar{\AL}^{j+1}) + D(\bar{\AL}^{j+1}) - D(\AL^{j+1}) \\
  &\geq \frac{\|\bx_{i(j)}\|^2}{2} \|\AL^j_{i(j)} - T_{i(j)}(\bw^j, \AL^j)\|^2 - \frac{L_{max}}{2}
  \| T_{i(j)}(\bw^j, \AL^j) - T_{i(j)}(\hat{\bw}^j, \AL^j)\|^2 
\end{align*}
So 
\begin{align*}
  E[D(\AL^j)] - E[D(\AL^{j+1})] &\geq \frac{R_{min}^2}{2n} E[ \|T(\bw^j, \AL^j)\|^2  ] 
  - \frac{L_{max}}{2n} E[ \|T(\hat{\bw}^j, \AL^j)-\AL^j\|^2 ] \\
  & \geq \frac{R_{min}^2}{2n} E[ \|T(\bw^j, \AL^j)-\AL^j\|^2] 
  - \frac{L_{max}}{2n}\frac{\tau^2 M^2 e^2}{n} E[\|T(\hat{\bw}^j, \AL^j)-\AL^j\|^2] \\
  & \geq \frac{R_{min}^2}{2n}E[\|T(\bw^j, \AL^j)-\AL^j\|^2] - \frac{2L_{max}}{2n}\frac{\tau^2 M^2 e^2}{n}
  (1+\frac{e\tau M}{\sqrt{n}}) E[\|T(\bw^j, \AL^j)-\AL^j\|^2] \\
  & \geq \frac{R_{min}^2}{2n}\left(1-\frac{2L_{max}}{R_{min}^2}(1+\frac{e\tau M}{\sqrt{n}}) (\frac{\tau^2 M^2 e^2}{n})\right) E[\|T(\bw^j, \AL^j)-\AL^j\|^2] 
\end{align*}
Let $b=(1-\frac{2L_{max}}{R_{min}^2}(1+\frac{e\tau M}{\sqrt{n}}) (\frac{\tau^2 M^2 e^2}{n}))$ and combine the above
inequality with eq~\eqref{eq:globalbound} we have
\begin{align*}
  E[D(\AL^j)]- E[D(\AL^{j+1})] & \geq b \kappa E[\|\AL^j - P_S(\AL^j)\|^2] \\
  & \geq \frac{b \kappa}{L_{max}} E[D(\AL^j)-D^*]. 
\end{align*}
Therefore, we have 
\begin{align*}
  E[D(\AL^{j+1})] - D^* &= E[D(\AL^j)] - (E[D(\AL^j)]-E[D(\AL^{j+1})]) - D^* \\
  &\leq (1-\frac{b\kappa}{L_{max}})(E[D(\AL^j)]-D^*). 
\end{align*}

\end{document}